%% file: NeurIPS2020_CameraReady_final.tex
\documentclass{article}
\pdfoutput=1
\usepackage[final,nonatbib]{neurips_2020}
\usepackage[utf8]{inputenc} % allow utf-8 input
\usepackage[T1]{fontenc}    % use 8-bit T1 fonts
\usepackage{hyperref}       % hyperlinks
\usepackage{url}            % simple URL typesetting
\usepackage{booktabs}       % professional-quality tables
\usepackage{amsfonts, amssymb}       % blackboard math symbols
\usepackage{nicefrac}       % compact symbols for 1/2, etc.
\usepackage{microtype}      % microtypography
\usepackage[colorinlistoftodos]{todonotes}
\usepackage{adjustbox}

\usepackage{pagecolor}
\usepackage{subcaption}

\title{The phase diagram of approximation rates for deep neural networks}

\author{% 
  Dmitry Yarotsky\\
  Skolkovo Institute of Science and Technology\\
  \texttt{d.yarotsky@skoltech.ru} \\
  \And
  Anton Zhevnerchuk\\
  Skolkovo Institute of Science and Technology\\
  \texttt{Anton.Zhevnerchuk@skoltech.ru} 
}

\usepackage[utf8]{inputenc} % allow utf-8 input
\usepackage[T1]{fontenc}    % use 8-bit T1 fonts
\usepackage{hyperref}       % hyperlinks
\usepackage{url}            % simple URL typesetting
\usepackage{booktabs}       % professional-quality tables
\usepackage{amsfonts}       % blackboard math symbols
\usepackage{nicefrac}       % compact symbols for 1/2, etc.
\usepackage{microtype}      % microtypography
\usepackage{lipsum}
\usepackage[english]{babel}
\usepackage{graphicx}
\usepackage{multicol}
\usepackage{bm}
\usepackage{amsmath}
\usepackage{amsthm}
\usepackage{pgfplots}
\usepackage{tikz-cd}
\pgfplotsset{compat=1.14}
\usepgfplotslibrary{fillbetween}
\usetikzlibrary{patterns}
\usetikzlibrary{shapes.misc}
\usetikzlibrary{arrows,arrows.meta,bending}
\usetikzlibrary{decorations.markings}
\usepackage{stackengine}
\usepackage{verbatim}

\newtheorem{proposition}{Proposition}[section]
\newtheorem{lemma}{Lemma}[section]
\newtheorem{theorem}{Theorem}[section]
\newtheorem{remark}{Remark}[section]
\newtheorem{corol}{Corollary}[section]

\newcommand{\R}{\mathbb{R}}
\newcommand{\Z}{\mathbb{Z}}

\newcommand{\n}{\mathbf{n}}
\newcommand{\m}{\mathbf{m}}
\newcommand{\x}{\mathbf{x}}
\newcommand{\q}{\mathbf{q}}
\renewcommand{\k}{\mathbf{k}}
\newcommand{\kh}{\widehat{\mathbf{k}}}

\newcommand{\ft}{\widetilde{f}}
\newcommand{\wt}{\widetilde{w}}

\renewcommand{\a}{\widehat{a}}
\newcommand{\at}{\widetilde{a}}

\tikzset{cross/.style={cross out, draw=black, minimum size=2*(#1), inner sep=0cm, outer sep=0cm}, cross/.default={1cm}}
\tikzset{
    set arrow inside/.code={\pgfqkeys{/tikz/arrow inside}{#1}},
    set arrow inside={end/.initial=>, opt/.initial=},
    /pgf/decoration/Mark/.style={
        mark/.expanded=at position #1 with
        {
            \noexpand\arrow[\pgfkeysvalueof{/tikz/arrow inside/opt}]{\pgfkeysvalueof{/tikz/arrow inside/end}}
        }
    },
    arrow inside/.style 2 args={
        set arrow inside={#1},
        postaction={
            decorate,decoration={
                markings,Mark/.list={#2}
            }
        }
    },
}

\graphicspath{{../figs/}}

\begin{document}
\maketitle

\begin{abstract}
We explore the phase diagram of approximation rates for deep neural networks and prove several new theoretical results. In particular, we generalize the existing result on the existence of deep discontinuous phase in ReLU networks to functional classes of arbitrary positive smoothness, and identify the boundary between the feasible and infeasible rates. Moreover, we show that all networks with a piecewise polynomial activation function have the same phase diagram. Next, we demonstrate that standard fully-connected architectures with a fixed width independent of smoothness can adapt to smoothness and achieve almost optimal rates. Finally, we consider deep networks with periodic activations (``deep Fourier expansion'') and prove that they have very fast, nearly exponential approximation rates, thanks to the emerging capability of the network to implement efficient lookup operations. 
\end{abstract}

\section{Introduction}

There is a subtle interplay between different notions of complexity for neural networks. One, most obvious, aspect of complexity is the network size measured in terms of the number of connections and neurons. Another is characteristics of the network architecture (e.g., shallow or deep). A third is the type of the activation function used in the neurons. Yet another, important but sometimes overlooked aspect is the precision of operations performed by neurons. All these complexities are connected by tradeoffs: if we fix a particular problem solvable by neural networks, then we have some freedom in decreasing one complexity at the cost of others. The question we address is: \emph{what are the limits of this freedom}? In the present paper we perform a systematic theoretical study of this question in the context of network expressiveness. We fix the classical approximation problem and explore the opportunities potentially present in solving it within different neural network scenarios. 

Specifically, suppose that we have a class $F$ of maps from the $d$-dimensional cube $[0,1]^d$ to $\mathbb R$,
and we want the network to approximate elements of $F$ in the uniform norm $\|\cdot\|_\infty$. We will make the standard assumption that $F$ is a Sobolev- or H\"older ball of smoothness $r>0$ (i.e., a ball of ``$r$ times differentiable functions'', see Section \ref{sec:prelim}). Then, for a particular type of approximation model, we examine the optimal \emph{approximation rate}, i.e. the relation beween the approximation accuracy and the required number $W$ of model parameters. Typically, this relation has the form of a power law
\begin{align}\label{eq:rate}
\|f - \ft_{W}\|_{\infty} = O(W^{-p}), \quad \forall f \in F,
\end{align}
where $\ft_W$ is an approximation of $f$ by a model with $W$ parameters, and $p$ is a constant (which we will also refer to as the \emph{rate}). In standard fully-connected networks, there is one parameter (weight) per each connection and neuron, so $W$ can be equivalently viewed as the size of the model. Our approach in this paper will be to analyze how the rates $p$ depend on various approximation conditions (e.g., network depth, activation functions, etc.). 

There are several important general ideas explaining which approximation rates $p$ we can reasonably expect in Eq.\eqref{eq:rate}. In the context of abstract approximation theory, we can forget (for a moment) about the network-based implementation of $\ft_{W}$ and just think of it as some approximate parameterization of $F$ by vectors $\mathbf w\in\mathbb R^W$. Let us view the approximation process $f\mapsto\ft_W$ as a composition of the \emph{weight assignment} map $f\mapsto \mathbf w_f\in \mathbb R^W$ and the \emph{reconstruction} map $\mathbf w_f\mapsto \ft_{W}\in\mathcal F,$ where $\mathcal F$ is the full normed space containing $F$. If both the weight assignment and reconstruction maps were linear, and so their composition $f\mapsto\ft_{W}$, the l.h.s. of Eq.\eqref{eq:rate} could be estimated by the \emph{linear $W$-width} of the set $F$ (see \cite{constrappr96}). For a Sobolev ball of $d$-variate functions $f$ of smoothness $r$, the linear $W$-width is asymptotically $\sim W^{-r/d}$, suggesting the approximation rate $p=\tfrac{r}{d}.$ Remarkably, this argument extends to \emph{non-linear} weight assignment and reconstruction maps under the assumption that the weight assignment is \emph{continuous}. More precisely, it was proved in \cite{continuous} that, under this assumption, $p$ in Eq.\eqref{eq:rate} cannot be larger than $\tfrac{r}{d}$.    

An even more important set of ideas is related to estimates of Vapnik-Chervonenkis dimensions of deep neural networks. The concept of expressiveness in terms of VC-dimension (based on finite set shattering) is weaker than expressiveness in terms of uniform approximation, but upper bounds on the VC-dimension directly imply upper bounds on feasible approximation rates. In particular, the VC-dimension of networks with piecewise-polynomial activations is $O(W^2)$ (\cite{goldberg1995bounding}), which implies that $p$ cannot be larger than $\tfrac{2r}{d}$ -- note the additional factor 2 coming from the power 2 in the VC bound.  We refer to the book \cite{anthony2009neural} for a detailed exposition of this and related results.

Returning to approximations with networks, the above arguments suggest that the rate $p$ in Eq.\eqref{eq:rate} can be up to $\frac{r}{d}$ assuming the continuity of the weight assignment, and up to $\frac{2r}{d}$ without assuming the continuity, but assuming a piecewise-polynomial activation function such as ReLU. We then face the constructive problem of showing that these rates can indeed be fulfilled by a network computation. One standard general strategy of proving the rate $p=\frac{r}{d}$ is based on polynomial approximations of $f$ (in particular, via the Taylor expansion).  A survey of early results along this line for networks with a single hidden layer and suitable activation functions can be found in \cite{pinkus1999review}. An interesting aspect of piecewise-linear activations such as ReLU is that the rate $p=\frac{r}{d}$ cannot be achieved with single-layer networks, but can be achieved with deeper networks implementing approximate multiplication and polynomials (\cite{yarsawtooth, liang2016why, petersen2018optimal, safran2017depth}).

%\pagecolor{yellow!30!orange}
\begin{figure}

\begin{subfigure}[b]{0.55\textwidth}
\centering
    \includegraphics[scale = 0.7, clip, trim=25mm 21mm 15mm 17mm ]{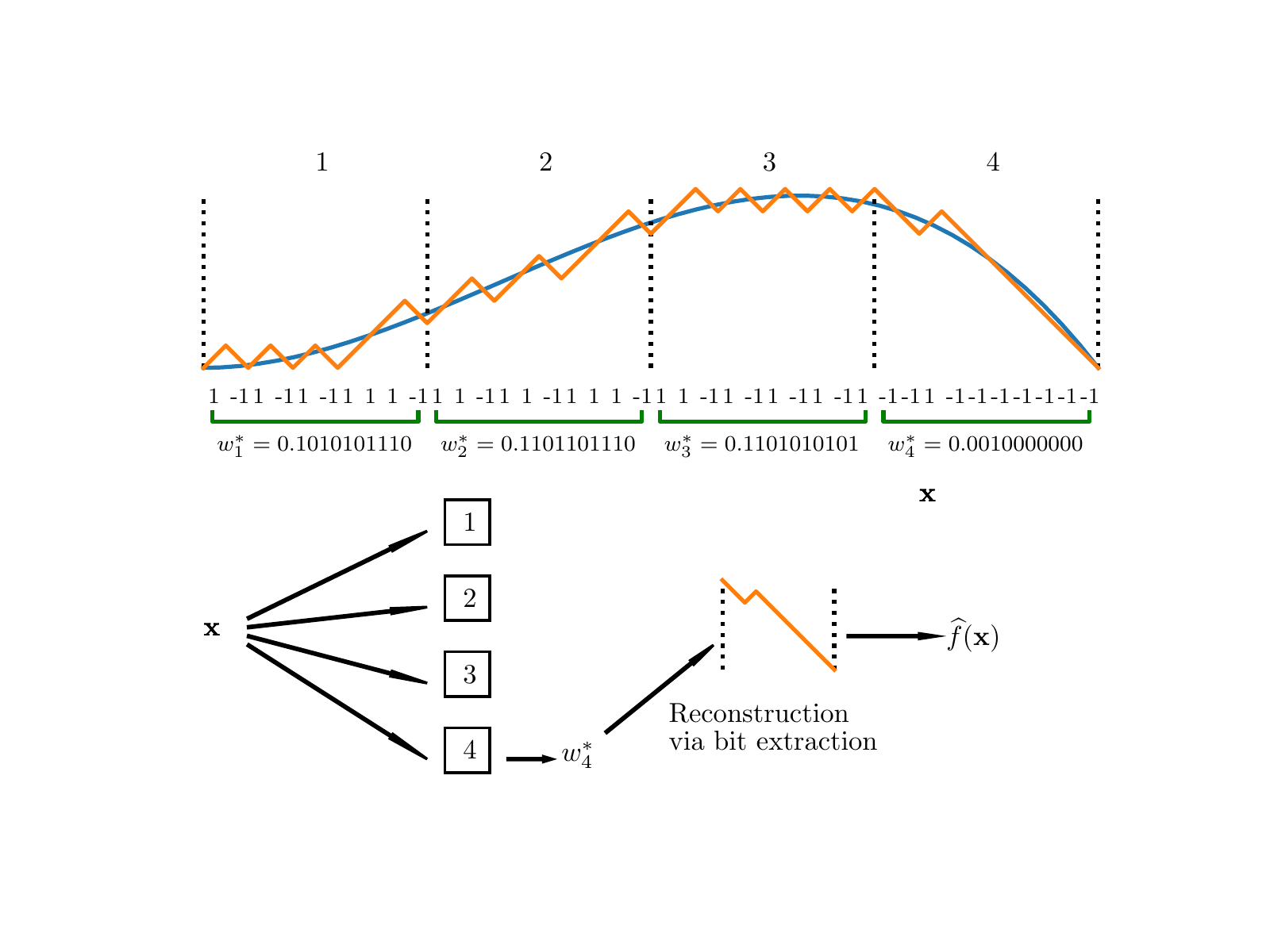}
    \caption{}
\end{subfigure}
\begin{subfigure}[b]{0.13\textwidth}
\hfill
\end{subfigure}
\begin{subfigure}[b]{0.3\textwidth}
\adjustbox{scale=1,right}{%
\begin{tikzcd}[scale=0.1, row sep=large]
&  w = 0.b_1b_2\ldots \arrow[dl, swap, "\lfloor 2w \rfloor"] \arrow[d,"2w-\lfloor 2w\rfloor"]
\\
b_1 & w_1=0.b_2b_3\ldots \arrow[dl, swap, "\lfloor 2w_1\rfloor"] \arrow[d,"2w_1-\lfloor 2w_1\rfloor"]
\\
b_2 & w_2=0.b_3b_4\ldots\arrow[dl, swap, "\lfloor 2w_2\rfloor"] \arrow[d,"2w_2-\lfloor 2w_2\rfloor"]
\\
\ldots & \ldots
\end{tikzcd}
}\caption{}
\end{subfigure}

\caption{\textbf{(a)} A high-rate approximation from  \cite{yaropt}. The domain $[0,1]^d$ is divided into  patches and an approximation to $f$ is encoded in each patch by a single network weight using a binary-type representation. Then, the network computes the approximation $\widetilde f(\mathbf x)$ by finding the relevant weight and decoding it using the bit extraction technique of \cite{bartlett1998almost}  (here, $d=1, r=1$, and $p=\tfrac{2r}{d}=2$). \textbf{(b)} Sequential bit extraction by a deep network \cite{bartlett1998almost}. (The floor function $\lfloor\cdot\rfloor$ can be approximated by ReLU with arbitrary accuracy via $\lfloor w\rfloor\approx \tfrac{1}{\delta}(w-1)_+-\tfrac{1}{\delta}(w-1-\delta)_+$ with a small $\delta$.)}\label{fig:approxdiscrete}

\end{figure}
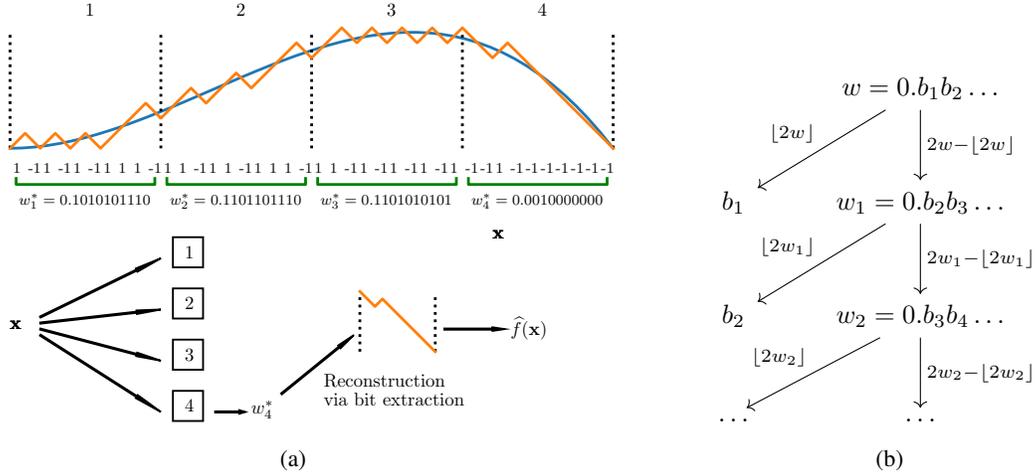

It was shown in \cite{yaropt} that ReLU networks can also achieve rates $p$ beyond $\frac{r}{d}.$ The result of \cite{yaropt} is stated in terms of the modulus of continuity of $f$; when restricted to H\"older functions with constant $r\le 1$, it implies that on such functions ReLU networks can provide rates $p$ in the interval $(\tfrac{r}{d},\tfrac{2r}{d}]$, in agreement with the mentioned upper bound $\tfrac{2r}{d}$. The construction is quite different from the case $p=\tfrac{r}{d}$ and has a ``coding theory'' rather than ``analytic'' flavor, see Fig.\ref{fig:approxdiscrete}. 
In agreement with continuous approximation theory and existing VC bounds, the construction inherently requires discontinuous weight assignment (as a consequence of coding finitely many values) and network depth (necessary for the bit extraction part). In this sense, at least in the case of $r\le 1$ one can distinguish two qualitatively different ``approximation phases'': the shallow continuous one corresponding to $p=\tfrac{r}{d}$ (and lower values), and the deep discontinuous one corresponding to $p\in(\tfrac{r}{d},\tfrac{2r}{d}]$.  It was shown in \cite{petersen2018optimal, voigtlaender2019approximation} that the shallow rate $p=\tfrac{r}{d},$ but not faster rates, can be achieved if the network weights are discretized with the precision of $O(\log(1/\epsilon))$ bits, where $\epsilon$ is the approximation accuracy.  

We remark in passing that in recent years there has also been a substantial amount of related research on other aspects of deep network expressiveness: e.g. (just to give a few examples) on performance scaling with input dimension \cite{poggio2017and, montanelli2019new}, depth separation \cite{telgarsky2016benefits, eldan2016power}, generalization from finite training sets \cite{schmidt2020nonparametric}, approximation on manifolds  \cite{shaham2018provable}, approximation of discontious functions \cite{petersen2018optimal} and specific signal structures \cite{perekrestenko2018universal, grohs2019deep}. These topics are outside the scope of the present paper.

\paragraph{Contribution of this paper.}
The developments described above leave many questions open. One immediate question is whether and how the deep discontinuous approximation phase generalizes to higher values of smoothness ($r> 1$). Another natural question is how much the network architectures providing the maximal rate $p=\tfrac{2r}{d}$ depend on the smoothness class. Yet another question is how sensitive the phase diagram is with respect to changing ReLU to other activation functions. In the present paper we resolve some of these questions and, moreover, offer new perspectives on the tradeoffs between different aspects of complexity in neural networks. Specifically:
\begin{itemize}
    \item In Section \ref{sec:phasediag}, we prove that the approximation phase diagram indeed generalizes to arbitrary smoothness $r>0$, with the deep discontinuous phase occupying the region $\tfrac{r}{d}<p\le \tfrac{2r}{d}$.   
    \item In Section \ref{sec:adaptivity}, we prove that the standard fully-connected architecture with a sufficiently large constant width $H$ only depending on the dimension $d$, say $H=2d+10$, can implement approximations that are asymptotically almost optimal %(with rate $p=\tfrac{2r}{d}$, up to a logarithmic correction)
    for \emph{arbitrary} smoothness $r$. This property can be described as  ``universal adaptivity to smoothness'' exhibited by such architectures.  
    \item In Section \ref{sec:otheractiv}, we discuss how the ReLU phase diagram can change if ReLU is replaced by other activation functions. In particular, we prove that the deep discontinuos phase can be constructed for any activation that has a point of nonzero curvature. This implies that the phase diagram for any piecewise polynomial activation is the same as for ReLU.
    \item In Section \ref{sec:deepfourier} we consider what we call \emph{``deep Fourier expansion''} -- approximation by a deep network with a periodic activation function, which can be seen as a generalization of the usual Fourier series approximation. We prove that such networks can provide much faster, exponential  rates compared to the polynomial rates of ReLU networks. The key element of the proof is a new version of the bit extraction procedure replacing sequential extraction by a dichotomy-based lookup.     
    \item In Section \ref{sec:info} we analyze the distribution of information in the networks implementing the discussed modes of approximation. In particular, we show that in the deep discontinuous ReLU phase the total information $\epsilon^{-d/r}$ is uniformly distributed over the $\epsilon^{-1/p}$ encoding weights, with $\epsilon^{1/p-d/r}$ bits per weight, while in the ``deep Fourier'' it is all concentrated in a single encoding weight.   
\end{itemize}

\section{Preliminaries}\label{sec:prelim}
\paragraph{Smooth functions.}
The paper revolves about what we informally describe as ``functions of smoothness $r$'', for any $r>0$. It is convenient to precisely define them in terms of  H\"older spaces. Any $r>0$ can be uniquely represented as $r=k+\alpha$  with an integer $k\ge 0$ and $0<\alpha\le 1$. We define the respective H\"older space $\mathcal C^{k,\alpha}([0,1]^d)$ as the space of $k$ times continuously differentiable functions on $[0,1]^d$ having a finite norm
\begin{align*}
    \|f\|_{\mathcal{C}^{k, \alpha}([0,1]^d)} = \max \Big\{ \max_{\k: |\k| \leq k} \max_{\mathbf x \in [0, 1]^d} |D^{\k} f(\mathbf x)|,  \max_{\k: |\k| = k} \sup_{\substack{\mathbf x, \mathbf y \in [0, 1]^d, \\ \mathbf x \neq \mathbf y}} \dfrac{|D^{\k} f(\mathbf x) - D^{\k}f(\mathbf y)|}{\|\mathbf x -\mathbf y\|^{\alpha}} \Big \}.
\end{align*}
Here $D^{\mathbf k}f$ denotes the partial derivative of $f$. We choose the sets $F$ appearing in Eq.\eqref{eq:rate} to be the unit balls in these H\"older spaces and denote them by $F_{r,d}$.

\paragraph{Neural networks.}
We consider conventional feedforward neural networks with layouts given by directed acyclic graphs. Each hidden unit performs a computation of the form $\sigma(\sum_{k=1}^K w_kz_k+h),$ where $z_k$ are the signals from the incoming connections, and $w_k$ and $h$ are the weights associated with this unit. In addition to input units and hidden units, the network is assumed to have a single output unit performing a computation similar to that of hidden units, but without the activation function. In Sections \ref{sec:phasediag} and \ref{sec:adaptivity} we assume that the activation function is ReLU: $\sigma(x)=a_+\equiv \max(0,a).$ In general, we refer to networks with an activation function $\sigma$ as \emph{$\sigma$-networks}.

In  general, we do not make any special connectivity assumptions about the network architecture. The exception is Section \ref{sec:adaptivity} where we consider a particular family of architectures in which  hidden units are divided into a sequence of layers, and each layer has a constant number of units. Two units are connected if and only if they belong to neighboring layers. The input units are connected to the units of the first hidden layer and only to them; the output unit is connected to the units of the last hidden layer, and only to them. We refer to this as a \emph{standard deep fully-connected architecture with constant width} (see Fig.\ref{fig:standard_net}). 

Whenever we mention a \emph{piecewise linear} or \emph{piecewise polynomial} activation, we mean that $\mathbb R$ can be divided into \emph{finitely many} intervals on which the activation is linear or polynomial (respectively). Without this condition of finiteness, activations could be made drastically more expressive (e.g., by joining a dense countable subset of polynomials \cite{maiorov1999lower}).

\paragraph{Approximations.} In the accuracy--complexity relation \eqref{eq:rate} we assume that approximations $\widetilde f_W$ are obtained by assigning $f$-dependent weights to a network architecture $\eta_W$ \emph{common} to all $f\in F.$ In particular, this allows us to speak of the weight assignment map $G_W: f\mapsto \mathbf w_f\in\mathbb R^W$ associated with a particular architecture $\eta_W.$ We say that the weight assignment is continuous if this map is continuous with respect to the topology of uniform norm $\|\cdot\|_\infty$ on $F.$ We will be interested in considering different approximation rates $p$, and we interpret Eq.\eqref{eq:rate} in a precise way by saying that a rate $p$ can be achieved iff 
\begin{align}\label{eq:rate2}
\inf_{\eta_W, G_W}\sup_{f\in F}\|f - \ft_{\eta_W,G_W}\|_{\infty} \le c_{F,p}W^{-p},
\end{align}
where $\ft_{\eta_W,G_W}$ denotes the approximation obtained by the weight assignment $G_W$ in the architecture $\eta_W$. Here and in  the sequel we generally denote by $c_{a,b,\ldots}$ various positive constants possibly dependent on $a,b,\ldots$ (typically on smoothness $r$ and dimension $d$). Throughout the paper, we will treat $r$ and $d$ as fixed parameters in the asymptotic accuracy-complexity relations.

\section{The phase diagram of ReLU networks}\label{sec:phasediag}

\begin{figure}
\centering
\minipage[b][][b]{0.4\textwidth}
\includegraphics[scale = 0.45, clip, trim=15mm 0mm 10mm 0mm ]{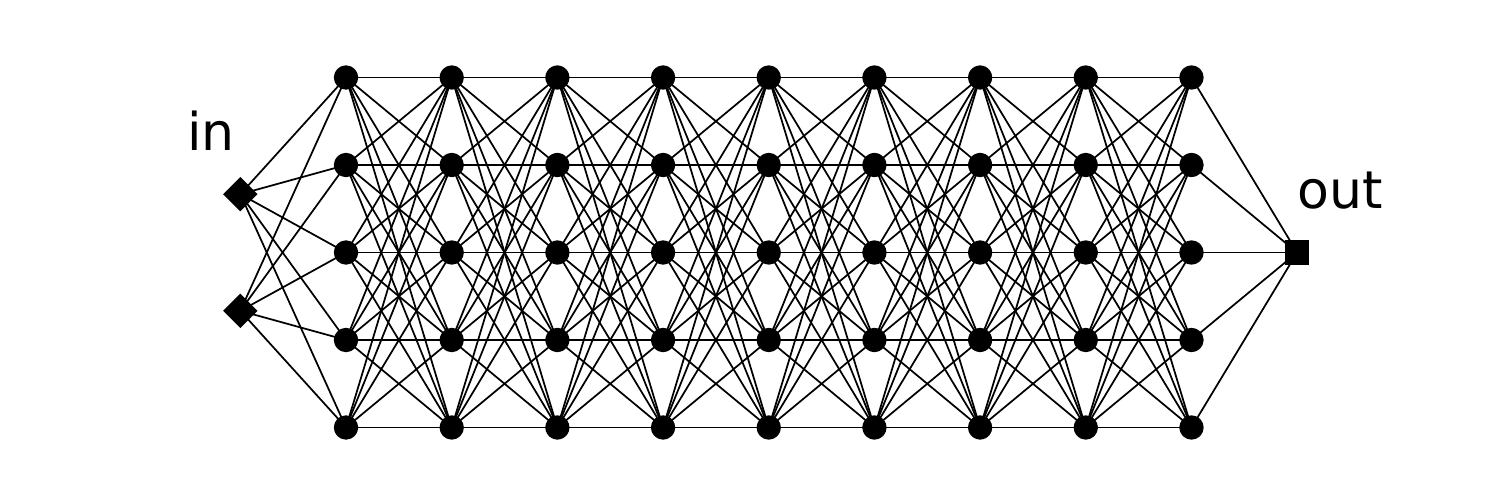}
\caption{A standard deep fully-connected architecture with  width 5.}\label{fig:standard_net}
\endminipage\hfill
\minipage[b][][b]{0.55\textwidth}
\centering
\input{phasediagram_final.tikz}
\caption{The phase diagram of approximation rates for ReLU networks. }\label{fig:phase_diagram}
\endminipage\hfill
\end{figure}

Our first main result is the phase diagram of approximation rates for ReLU networks, shown in Fig.\ref{fig:phase_diagram}. The  ``shallow continuous phase'' corresponds to $p=\tfrac{r}{d}$, the ``deep discontinuous phase'' corresponds to $\tfrac{r}{d}<p\le\tfrac{2r}{d}$, and the infeasible region corresponds to $p>\tfrac{2r}{d}.$ Our main new contribution is the exact location of the deep discontinuous phase for all $r>0$. The precise meaning of the diagram is explained by the following series of theorems (partly established in earlier works). 

\begin{theorem}[The shallow continuous phase]
The approximation rate $p = \frac{r}{d}$ in Eq.\eqref{eq:rate2} can be achieved by ReLU networks having $L \le c_{r,d} \log W$ layers, and with a continuous weight assignment.
\end{theorem}
This result was proved in \cite{yarsawtooth} in a slightly weaker form, for integer $r$ and with error $O(W^{-r/d} \log^{r/d} W)$ instead of $O(W^{-r/d})$. The proof is based on ReLU approximations of local Taylor expansions of $f$. The extension to non-integer $r$ is immediate thanks to our definition of general $r$-smoothness in terms of H\"older spaces. The logarithmic factor $\log^{r/d} W$ can be removed by observing that the computation of the approximate Taylor polynomial can be isolated from determining its coefficients and hence only needs to be implemented once in the network rather than for each local patch as in \cite{yarsawtooth} (see Remark \ref{rm:nolog}; the idea of isolation of operations common to all patches is developed much further in the proof Theorem \ref{th:deepphase} below, and is applicable in the special case $p=\tfrac{r}{d}$).  

\begin{theorem}[Feasibility of rates $p>\tfrac{r}{d}$]\label{th:feasible}{}\hfill
\begin{enumerate}
\item Approximation rates $p > \frac{2r}{d}$ are infeasible for networks with piecewise-polynomial activation function and, in particular, ReLU networks;
\item Approximation rates $p \in (\frac{r}{d}, \frac{2r}{d}]$ cannot be achieved with continuous weights assignment;
\item If an approximation rate $p \in (\frac{r}{d}, \frac{2r}{d}]$ is achieved with ReLU networks, then the number of layers $L$ in $\eta_W$ must satisfy $L \geq c_{p,r,d} W^{pd/r - 1} / \log W$ for some $c_{p,r,d} > 0$.
\end{enumerate}
\end{theorem}
These statements follow from existing results on continuous nonlinear approximation (\cite{continuous} for statement 2) and from upper bounds on VC-dimensions of neural networks (\cite{goldberg1995bounding} for statement 1 and \cite{bartlett2019nearly} for statement 3), see \cite[Theorem~1]{yarsawtooth} for a derivation. The extensions to arbitrary $r$ are straightforward.

The main new result in this section is the existence of approximations with $p\in(\tfrac{r}{d}, \tfrac{2r}{d}]$:
\begin{theorem}[The deep discontinuous phase]\label{th:deepphase} For any $r>0,$ any rate $p \in (\frac{r}{d}, \frac{2r}{d}]$ can be achieved with deep ReLU networks with $L \le c_{r,d} W^{pd/r - 1}$ layers.
\end{theorem}
This result was proved in \cite{yaropt} in the case $r\le 1$. We generalize this to arbitrary $r$ by combining the coding-based approach of \cite{yaropt} with Taylor expansions. 
We give a sketch of proof below; the full proof is given in Section \ref{sec:proofdeepphase}.

\emph{Sketch of proof.} We use two length scales for the approximation:  the coarser one $\tfrac{1}{N}$ and the finer one $\tfrac{1}{M}$, with $M\gg N.$ We start by partitioning the cube $[0,1]^d$ into $\sim N^d$ patches (particularly, simplexes) of linear size $\sim\tfrac{1}{N},$ and then sub-partitioning them into patches of linear size $\sim\tfrac{1}{M}.$ In each of the finer $M$-patches $\Delta_M$ we approximate the function $f\in F_{r,d}$ by a Taylor polynomial $P_{\Delta_M}$ of degree $\lceil r\rceil-1.$  Then, from the standard Taylor remainder bound, we have $|f(\mathbf x)-P_{\Delta_M}(\mathbf x)|=O(M^{-r})$ on $\Delta_M$. This shows that if $\epsilon$ is the required approximation accuracy, we should choose $M\sim \epsilon^{-1/r}.$

Now, if we tried to simply save the Taylor coefficients for each $M$-patch in the weights of the network, we would need at least $\sim M^{d},$ i.e. $\sim\epsilon ^{-d/r}$, weights in total. This corresponds to the classical rate $p=\tfrac{r}{d}$. In order to save on the number of weights and achieve higher rates, we collect Taylor coefficients of all $M$-patches lying in one $N$-patch and encode them in a single \emph{encoding weight} associated with this $N$-patch. Given $p>\tfrac{r}{d},$ we choose $N\sim \epsilon^{-1/(pd)},$ so that in total we create $\sim \epsilon^{-1/p}$ encoding weights, each containing information about $\sim (M/N)^d$, i.e. $\sim\epsilon^{-(d/r-1/p)}$, Taylor coefficients. The number of encoding weights then matches the desired complexity $W\sim\epsilon^{-1/p}$. 

To encode the Taylor coefficients we actually need to discretize them first. Note that to reconstruct the Taylor approximation in an $M$-patch with accuracy $\epsilon,$ we need to know the Taylor coefficients of order $k$ with precision $\sim M^{-(r-k)}$. We implement an efficient sequential encoding/decoding procedure for the approximate Taylor coefficients of orders $k<\lceil r\rceil$ for all $M$-patches lying in the given $N$-patch $\Delta_N$. Specifically, choose some sequence $(\Delta_{M})_t$ of the $M$-patches in $\Delta_N$ so that neighboring elements of the sequence correspond to neighboring patches. Then, the order-$k$ Taylor coefficients at $(\Delta_{M})_{t+1}$ can be determined with precision $\sim M^{-(r-k)}$ from the respective and higher order coefficients at $(\Delta_{M})_{t}$ using $O(1)$ predefined discrete values. This allows us to encode all the approximate Taylor coefficients in all the $M$-patches of $\Delta_N$ by a single $O((M/N)^d)$-bit number. 

To reconstruct the approximate Taylor polynomial for a particular input $\mathbf x\in\Delta_M\subset\Delta_N$, we sequentially reconstruct all the coefficients for the sequence $(\Delta_{M})_t$, and, among them, select the coefficients at the patch $(\Delta_{M})_{t_0}=\Delta_M$.  The sequential reconstruction can be done by a deep subnetwork with the help of the bit extraction technique \cite{bartlett1998almost}. The depth of this subnetwork is proportional to the number of $M$-patches in $\Delta_N$, i.e. $\sim (M/N)^d$, which is $\sim \epsilon^{-(d/r-1/p)}$ according to our definitions of $N$ and $M$. If $p\le\tfrac{2r}{d},$ then $\tfrac{d}{r}-\tfrac{1}{p}\le \tfrac{1}{p}$ and hence this depth is smaller or comparable to the number of encoding weights, $\epsilon^{-1/p}.$ However, if $p>\tfrac{2r}{d},$ then the depth is asymptotically larger than the number of encoding weights, so the total number of weights is dominated by the depth of the decoding subnetwork, which is $\gtrsim \epsilon^{-d/(2r)}$, and the approximation becomes less efficient than at $p=\tfrac{2r}{d}$. This explains why $p=\tfrac{2r}{d}$ is the boundary of the feasible region.

Once the (approximate) Taylor coefficients at $\Delta_M\ni \mathbf x$ are determined, an approximate  Taylor polynomial $\widetilde P_{\Delta_M}(\mathbf x)$ can be computed by a ReLU subnetwork implementing efficient approximate multiplications \cite{yarsawtooth}. \qed

\section{Fixed-width networks: universal adaptivity to smoothness}\label{sec:adaptivity}
The network architectures constructed in the proof of Theorem \ref{th:deepphase} to provide the faster rates $p\in(\tfrac{r}{d},\tfrac{2r}{d}]$ are relatively complex and $r$-dependent. We can ask if such rates can be supported by some simple conventional architectures. It turns out that we can achieve nearly optimal rates using standard fully-connected architectures with sufficiently large constant widths only depending on $d$:
\begin{theorem}\label{th:constwidth}
Let $\eta_W$ be standard fully-connected ReLU architectures with width $2d+10$ and $W$ weights. Then 
\begin{equation}\label{eq:constwidth}
\inf_{G_W}\sup_{f\in F_{r,d}}\|f-\ft_{\eta_W,G_W}\|_\infty\le c_{r,d} W^{-2r/d}\log^{2r/d} W.
\end{equation}
\end{theorem}
The rate in Eq.\eqref{eq:constwidth} differs from the optimal rate with $p=\tfrac{2r}{d}$ only by the logarithmic factor $\log^{2r/d} W$.
We give a sketch of proof of Theorem \ref{th:constwidth} below, and details are provided in Section \ref{sec:proofconstwidth}.

An interesting result proved in \cite{hanin2017approximating, lu2017expressive} (see also \cite{lin2018resnet} for a related result for ResNets) states that standard fully-connected ReLU architectures with a fixed width $H$ can approximate any $d$-variate continuous function if and only if $H\ge d+1$. Theorem \ref{th:constwidth} shows that with slightly larger widths, such networks can not only adapt to any function, but also adapt to its smoothness. The results of \cite{hanin2017approximating, lu2017expressive} also show that Theorem \ref{th:constwidth} cannot hold with $d$-independent widths.

\emph{Sketch of proof of Theorem \ref{th:constwidth}.} The proof is similar to the proof of Theorem \ref{th:deepphase}, but requires a different implementation of the reconstruction of $\widetilde f(\mathbf x)$ from encoded Taylor coefficients. The network constructed in Theorem \ref{th:deepphase} traverses $M$-knots of an $N$-patch and computes Taylor coefficients at the new $M$-knot by updating the coefficients at the previous $M$-knot. This computation can be arranged within a fixed-width network, but its width depends on $r$, since we need to store the coefficients from the previous step, and the number of these coefficients grows with $r$ (see \cite{yaropt} for the constant-width fully-connected implementation in the case of $r\le 1,$ in which the Taylor expansion degenerates into the 0-order approximation). 

To implement the approximation using an $r$-independent network width, we can decode the Taylor coefficients afresh at each traversed $M$-knot, instead of updating them. This is slightly less efficient and leads to the additional logarithmic factor in Eq.\eqref{eq:constwidth}, as can be seen in the following way. First, since we need to reconstruct the Taylor coefficients of degree $k$ with precision $O(M^{-(r-k)}),$ we need to store $\sim \log M$ bits for each coefficient in the encoding weight. Since $M\sim \epsilon^{-1/r},$ this means a $\sim\log (1/\epsilon)$-fold increase in the depth of the decoding subnetwork. Moreover, an approximate Taylor polynomial must be computed separately for each $M$-patch. Multiplications   can be implemented with accuracy $\epsilon$ by a fixed-width ReLU network of depth $\sim(\log(1/\epsilon))$ (see \cite{yarsawtooth}). Computation of an approximate polynomial of the components of the input vector $\mathbf x$ can be arranged as a chain of additions and multiplications in a network of constant width independent of the degree of the polynomial -- assuming the coefficients of the polynomial are decoded from the encoding weight and supplied as they become required. This shows that we can achieve accuracy $\epsilon$ with a network of constant width independent of $r$ at the cost of taking the larger depth $\sim\epsilon ^{-d/(2r)}\log(1/\epsilon)$ (instead of simply $\sim\epsilon ^{-d/(2r)}$ as in Theorem \ref{th:deepphase}). Since $W$ is proportional to the depth, we get $W\sim \epsilon ^{-d/(2r)}\log(1/\epsilon)$. By inverting this relation, we obtain Eq.\eqref{eq:constwidth}.  \qed

\section{Activation functions other than ReLU}\label{sec:otheractiv}
We discuss now how much the ReLU phase diagram of Section \ref{sec:phasediag} can change if we use more complex activation functions. We note first that statement 1 of Theorem  \ref{th:feasible} holds not only for ReLU, but for any piecewise-polynomial activation functions, so that the region $p>\tfrac{2r}{d}$ remains infeasible for any such activation. Also, since all piecewise-linear activation functions are essentially equivalent (see e.g. \cite[Proposition 1]{yarsawtooth}), the phase diagram for any piecewise-linear activation is the same as for ReLU. 

Our main result in this section states that Theorem \ref{th:deepphase} establishing the existence of the deep discontinuous phase remains valid for any activation  that has a point of nonzero curvature.

\begin{theorem}\label{th:deepsecond} Suppose that the activation function $\sigma$ has a point $x_0$ where the second derivative $\tfrac{d^2\sigma}{dx^2}(x_0)$ exists and is nonzero. Then,  any rate $p \in (\frac{r}{d}, \frac{2r}{d})$ can be achieved with deep $\sigma$-networks with $L \le c_{r,d} W^{pd/r - 1}$ layers.
\end{theorem} 
The proof is given in Section \ref{sec:deepsecondproof}; its idea is to reduce the approximation by $\sigma$-networks to deep polynomial approximations. Then, we can follow the lines of the proof of Theorem \ref{th:deepphase} with some adjustments (in particular, we replace the usual bit extraction dynamic as in Fig.\ref{fig:approxdiscrete}(b)  by a polynomial dynamical system). We remark that in general, if constrained by degree, polynomials poorly approximate ReLU and other piecewise linear functions \cite{telgarsky2017neural}, but in our setting the polynomials are constrained by their compositional complexity rather than degree, in which case a polynomial approximation of ReLU can be much more accurate.

Combined with Statement 1 of Theorem  \ref{th:feasible}, Theorem \ref{th:deepsecond} implies, in particular, that the phase diagram for general piecewise polynomial activation functions is the same as for ReLU: 

\begin{corol} Let $\sigma$ be a continuous piecewise polynomial activation function. Then the rates $p<\tfrac{2r}{d}$ are feasible for $\sigma$-networks, and the rates $p>\tfrac{2r}{d}$ are infeasible.
\end{corol} 

A remarkable class of functions that can be seen as a far-reaching generalization of polynomials are the Pfaffian functions \cite{khovanskii}. Level sets of these functions admit bounds on the number of their connected components that are similar to analogous bounds for algebraic sets, and this is a key property in establishing upper bounds on VC dimensions of networks. In particular, it was proved in \cite{karpinski1997polynomial} that the VC-dimension of networks with the standard sigmoid activation function $\sigma(x)=1/(1+e^{-x})$ is upper-bounded by $O(W^2k^2),$ where $k$ is the number of computation units (see also \cite[Theorem 8.13]{anthony2009neural}). Since $k\le W$, the bound $O(W^2k^2)$ implies the slightly weaker bound $O(W^4)$. Then, by mimicking the proof of statement 1 of Theorem \ref{th:feasible} and replacing there the bound $O(W^2)$ for piecewise-polynomial activation by the bound $O(W^4)$ for the standard sigmoid activation, 
we get
\begin{theorem}\label{th:sigmoid} For networks with the standard sigmoid activation function $\sigma=1/(1+e^{-x})$, the rates $p>\tfrac{4r}{d}$ are infeasible.
\end{theorem} 
It appears that there remains a significant gap between the upper and lower VC dimension bounds for networks with $\sigma(x)=1/(1+e^{-x})$ (see a discussion in \cite[Chapter 8]{anthony2009neural}). Likewise, we do not know if the approximation rates up to $p= \frac{4r}{d}$ are indeed feasible with this $\sigma$.

All the above results ignore both precision and magnitude of the network weights. In fact, the rates $p>\tfrac{2r}{d}$ can be excluded for rather general activation functions if we put some mild constraints on the growth of the weights. In Section \ref{sec:weightbounds} we explain this point using a covering number bound from \cite[Theorem 14.5]{anthony2009neural}.

\section{``Deep Fourier expansion''}\label{sec:deepfourier}

Note that the usual Fourier series expansion $f(\mathbf x)\sim\sum_{\mathbf n\in\mathbb Z^d}a_{\mathbf n}e^{2\pi i\mathbf n\cdot\mathbf x}$ for a function $f$ on $[0,1]^d$ can be viewed as a neural network with one hidden layer, the $\sin$ activation function, and predefined weights in the first layer. Standard convergence bounds for Fourier series (see e.g. \cite{jackson1930theory}) correspond to the shallow continuous rate $p=\tfrac{r}{d}$, in agreement with the linearity of the standard assignment of Fourier coefficients. We can ask what happens to the expressiveness of this approximation if we generalize it by removing all constraints on the architecture and weights,  i.e., consider a general deep network with the $\sin$ activation function.

It turns out that such a  model is drastically more expressive than both standard Fourier expansion and deep ReLU networks. The key factor in this is the \emph{periodicity} of the activation function $\sigma=\sin$; the particular form of $\sigma$ is not that important. Our main result below assumes that the network can use both ReLU and $\sigma$ as activation functions; we refer to these networks as \emph{mixed ReLU/$\sigma$} networks. 

\begin{theorem}\label{th:sin}
Fix $r,d$. Let $\sigma:\mathbb R\to\mathbb R$ be a Lipschitz periodic function with period $T$. Suppose that $\sigma(x)>0$ for $x\in(0,{T}/{2})$ and $\sigma(x)<0$ for $x\in({T}/{2},T)$, and also that $\max_{x\in\mathbb R}\sigma(x)=-\min_{x\in\mathbb R}\sigma(x).$  Then: 
\begin{enumerate}
\item For any number $W$, we can find a mixed ReLU/$\sigma$ network architecture $\eta_W$ with $W$ weights, and a corresponding weight assignment $G_W$, such that
\begin{align}\label{eq:ratesin}
\sup_{f\in F_{r,d}}\|f - \ft_{\eta_W,G_W}\|_{\infty} \le \exp\big(-c_{r,d}W^{1/2}\big)
\end{align}
with some $r,d$-dependent constant $c_{r,d}>0.$
\item Moreover, the above architecture $\eta_W$ has only one weight whose value depends on $f\in F_{r,d}$; for all other weights the assignment $G_W$ is $f$-independent.
\end{enumerate}
\end{theorem}
In contrast to the previously considered power law rates \eqref{eq:rate2}, the rate \eqref{eq:ratesin} is exponential and corresponds to $p=\infty$, so that the ReLU-infeasible sector $p>\tfrac{2r}{d}$ is fully feasible for mixed ReLU/periodic networks. Moreover, statement 2 of the above theorem means that all information about the approximated function $f$ can be encoded in a single network weight.

\begin{figure}
\begin{subfigure}[b]{0.45\textwidth}
\centering
\includegraphics[scale = 0.4, clip, trim=0mm 60mm 0mm 40mm ]{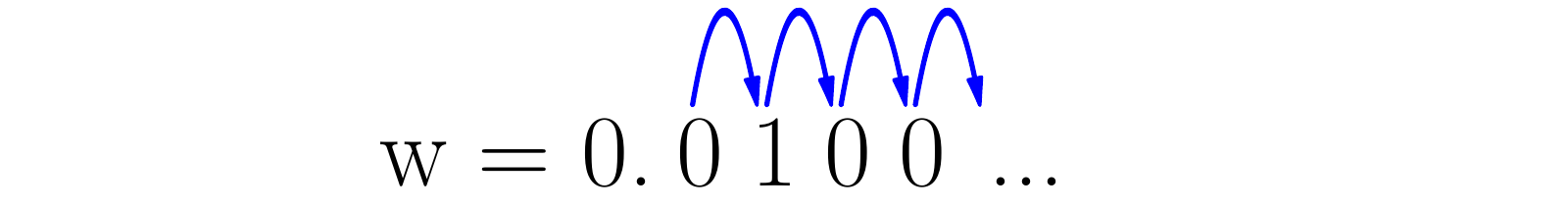}
    \caption{With ReLU}
\end{subfigure}
\begin{subfigure}[b]{0.5\textwidth}
\adjustbox{scale=1,right}{%
\includegraphics[scale = 0.4, clip, trim=0mm 14mm 0mm 80mm ]{sinBitExtraction.pdf}
}
\caption{With a periodic activation $\sigma$}
\end{subfigure}
\caption{Standard sequential (a) and dichotomy-based (b) bit extraction. Bit extraction is used to decode information from network weights and is crucial in achieving non-classical rates $p>\tfrac{r}{d}.$ Standard bit extraction (\cite{bartlett1998almost}, see Fig.\ref{fig:approxdiscrete}) is available with the threshold or ReLU activation functions. The bits are decoded one-by-one, which requires a significant networks depth and caps feasible rates at $p=\tfrac{2r}{d}.$ In contrast, ``deep Fourier expansion'' of Theorem \ref{th:sin} is essentially based on a more efficient dichotomy-based lookup that becomes available if neurons can implement a periodic activation function (see Section \ref{sec:sketchsin}).}\label{fig:sinextraction}
\end{figure}   

The sketch of proof of Theorem \ref{th:sin} is given in Section \ref{sec:sketchsin}, and details are provided in Section \ref{sec:proofsin}. The main idea of the network design is to compute each digit of the output using a dynamical system controlled by the digits of the input.
The faster rate can be interpreted as resulting from an efficient, dichotomy-based lookup that can be performed in networks including both ReLU and a periodic activation, see Fig.\ref{fig:sinextraction}. 

It is well-known that some exotic activation functions allow to achieve rates even higher than those we have discussed. For example, a result of \cite{maiorov1999lower} based on the Kolmogorov Superposition Theorem  (\cite[p. 553]{constrappr96}) shows the existence of a strictly increasing analytic activation function $\sigma$ such that any $f\in C([0,1]^d)$ can be approximated with arbitrary accuracy by a three-layer $\sigma$-network with only $9d+3$ units. However, in contrast to these results, our Theorem \ref{th:sin} holds for a very simple and general class of activation functions.

\section{Distribution of information in the network}\label{sec:info}

\begin{table}
\begin{center}
\renewcommand{\arraystretch}{1.5}
\begin{tabular}{lccc}
\toprule
\textrm{Approximation} & \textrm{Shallow ReLU  } & \textrm{Deep ReLU}& \textrm{``Deep Fourier''} 
\\
\midrule 
Rate ($p$) & $p=\tfrac{r}{d}$ &  $p\in(\tfrac{r}{d},\tfrac{2r}{d}]$ & $p=\infty$\\
Weight assignment & continuous & discontinuous & discontinuous \\
Network depth ($L$) & $\log(1/\epsilon)$ & $\epsilon^{1/p-d/r}$ & $\log(1/\epsilon)$\\
Number of weights, total ($W$) 
 & $\epsilon^{-d/r}$ & $\epsilon^{-1/p}$ & $\log^2(1/\epsilon)$\\
Number of encoding weights & $\epsilon^{-d/r}$ & $\epsilon^{-1/p}$ & 1 \\
Bits / encoding weight & $\log(1/\epsilon)$ & $\epsilon^{1/p-d/r}$ & $\epsilon^{-d/r}\log(1/\epsilon)$  \\
\bottomrule
\end{tabular}
\end{center}
\caption{Summary of the examined approximation modes. $\epsilon$ stands for the approximation accuracy $\|f-\widetilde f\|_\infty$  achieved uniformly on the H\"older ball $F_{r,d}$. The expressions in the bottom four rows show the orders of magnitude for various network characteristics w.r.t. $\epsilon$.}\label{tab:info}
\end{table}

It is interesting to examine how information about the approximated function $f$ is distributed in the network (see Table \ref{tab:info}). The classical theorem of Kolmogorov \cite{KolmogorovTikhomirov} shows that the $\epsilon$-entropy of the H\"older ball $F_{r,d}$ scales as $\epsilon^{-d/r}$ at small $\epsilon.$ This means that any family of networks achieving accuracy $\epsilon$ on this ball must include at least $\epsilon^{-d/r}$ bits of information about $f\in F_{r,d}.$ This imposes constraints on the magnitude and/or precision of network weights: if the network is small and the weights have a limited space of values, the network simply cannot contain the necessary amount of information  (\cite{bolcskei2017memory,petersen2018optimal, voigtlaender2019approximation}). 

Classical linear models or ``weakly nonclassical'' models such as shallow ReLU networks contain $\epsilon^{-d/r}$ weights, and a weight precision of $O(\log (1/\epsilon))$ bits is sufficient to accomodate the total $\epsilon$-entropy $\epsilon^{-d/r}$ (\cite{voigtlaender2019approximation}). In contrast, the models in the ``deep discontinuous ReLU'' phase contain much fewer weights and accordingly need a much higher weight precision. Specifically, it follows from the proofs of Theorems \ref{th:deepphase} and \ref{th:deepsecond} that the number of encoding weights in a network with rate $p\in (\tfrac{r}{d},\tfrac{2r}{d}]$ is $\sim\epsilon^{-1/p}$, while each encoding weight must be specified with accuracy $c^{\epsilon^{1/p-d/r}}$ with some constant $c>0,$ i.e. must have $\sim\epsilon^{1/p-d/r}$ bits.      

In the ``deep Fourier'' model, the encoding weight is unique. In the end of Section \ref{sec:sketchsin} we roughly estimate the information contained in this weight as $\epsilon^{-d/r}\log(1/\epsilon)$, again in agreement with the $\epsilon$-entropy $\epsilon^{-d/r}$ of the H\"older ball $F_{r,d}$.

\section{Discussion}
Our results highlight tradeoffs between complexity of the network size and complexity of activations and/or arithmetic operations: the size can be decreased substantially at the cost of the other complexities. In addition to the increased precision of network operations, this requires the weight assignment to be discontinuous with respect to the fitted function $f$. While we do not discuss learning aspects in this paper, this discontinuity suggests that such networks should be hard to train by usual gradient-based methods, and would probably require other types of fitting algorithms.

The mentioned complexity tradeoffs are not unlimited: we have shown that for all piecewise polynomial activations the feasible rates span the sector $p\le \tfrac{2r}{d}$. We do not know if this remains true for other standard nonpolynomial activations such as the standard sigmoid. This question seems to be essentially rooted in the optimality of Khovanskii's fewnomial bounds, which is a long-standing problem in algebraic geometry \cite{haas2002simple, dickenstein2007extremal}.

We have introduced the ``deep Fourier'' model -- a hypothetical computational model assuming that the neurons can perfectly compute a periodic function of their inputs. This model allows to achieve exponential approximation rates while storing all information in a single weight. This result is purely theoretical; it doesn't seem possible to implement such a model using practical technologies. Rather, we see the main interest of this result in the theoretical demonstration of a huge network size reduction compared to the usual shallow Fourier expansion, and in the associated novel bit extraction mechanism.

\section{Broader impact}
Not applicable.

\section{Acknowledgments and Funding Transparency Statement}
We thank  Christoph Schwab for suggesting an extension of Theorem \ref{th:sin} to general periodic activations. We also thank the anonymous reviewers for several useful comments and suggestions. The research was not supported by third parties. The authors are not aware of any conflict of interest associated with this research.

\bibliographystyle{unsrt}
\bibliography{references}

\appendix
\include{appendix}

\end{document}

%% file: phasediagram_final.tikz
\begin{tikzpicture}[scale=0.8]
\begin{axis}
[
    xmin=0, xmax=14,
    ymin=0, ymax=8,
    xlabel={$p$ (rate)},
    ylabel={$r$ (smoothness)},
    xtick=\empty,
    ytick=\empty,
   width=15cm,
   height=8cm,
   scale=0.6
]

\path[draw=black,fill=green,line join=round,line cap=round,miter limit=10.00,line
    width=1.152pt] (0,0) -- (15,10) -- (0,10) -- (0,0);

\path[draw=black,fill=green!20,line join=round,line cap=round,miter limit=10.00,line
    width=1.152pt] (0,0) -- (18,12) -- (18,6) -- (0,0);
    
\path[draw=black,fill=red!50,line join=round,line cap=round,miter limit=10.00,line
    width=1.152pt] (0,0) -- (18,6) -- (18,0) --(0,0);

\node[align=center, text width=5cm, %font=\footnotesize
] at (11.3,5.5) {Deep NN\\ $L \sim W^{pd/r - 1}$ \\ discont. WA};
\node[align=center, text width=5cm, font=\footnotesize] at (3.5,6.1) {Shallow NN, \\ continuous \\ weight assignment};
\node[align=center, font=\footnotesize] at (12.0,2.5) {Infeasible};
\node[coordinate, pin={[pin edge={black}]350:{\color{black}$p=\frac{2r}{d}$}}] at (axis cs:6.0,2.0){};
\node[coordinate, pin={[pin edge={black}]180:{\color{black}$p=\frac{r}{d}$}}] at (axis cs:4.5,3){};
\end{axis}
\end{tikzpicture}

%% file: appendix.tex
\section{Theorem \ref{th:deepphase}: proof details}\label{sec:proofdeepphase}
We follow the paper \cite{yaropt} where Theorem \ref{th:deepphase} was proved for $r\le 1$, and generalize it to arbitrary $r>0$ using the strategy explained in Section \ref{sec:phasediag}. 
Given $p \in (\frac{r}{d}, \frac{2r}{s}]$ we show that it is possible to construct a network architecture with $W$ weights and $L = O(W^{pd/r - 1})$ layers which approximates every $f \in {F}_{r, d}$ with error $O(W^{-p})$. In \autoref{rm:nolog} we deal with the case $p = \frac{r}{d}$.

We start by describing the space partition and related constructions. Then we give an overview of the network structure. Finally, we describe in more detail the network computation of the Taylor approximations, which is the main novel element of Theorem \ref{th:deepphase}.

\subsection{Space partitions}

For an integer $N \geq 1$ we denote by $\mathcal{P}_N$ a standard triangulation of $\R^d$ into simplexes:
\begin{align*}
\Delta_{N, \n, \rho} = \left\{\mathbf x \in \R^d : 0 \leq x_{\rho(1)} - \frac{\n_{\rho(1)}}{N} \leq \dots \leq x_{\rho(d)} - \frac{\n_{\rho(d)}}{N} \right\},
\end{align*}
where $\n \in \Z^d$ and $\rho$ is a permutation of $d$ elements. The vertices of these simpixes are the points of the grid $(\Z / N)^d$. We call the set of all the vertices \emph{the $N$-grid} and a particular vertex \emph{an $N$-knot}. For an $N$-knot we call the union of simplexes it belongs to \emph{an $N$-patch}. We denote a set of all $N$-knots $\mathbf{K}_{N}$.

Let $\phi: \R^d \to \R$ be the ``spike'' function defined as the continuous piecewise linear function such that:
\begin{enumerate}
\item $\phi$ is linear on every simplex from the triangulation $\mathcal{P}_1$;
\item $\phi(0) = 1$, $\phi(\n) = 0$ for all other $\n \in \Z^d$.
\end{enumerate}
The function $\phi(\mathbf x)$ can be computed by a feed-forward ReLU network with $O(d^2)$ weights (see \cite[Section~4.2]{yaropt} for details). We treat $d$ as a constant, so we can say that $\phi(\mathbf x)$ can be computed by a network with a constant number of weights. Note that for integer $N$ and $\n \in Z^d \cap [0, N]^d$, the function $\phi(N\mathbf x - \n)$ is a continuous piecewise linear function which is linear in each simplex from $\mathcal{P}_N$, is equal to 1 at $\mathbf x=\frac{\n}{N}$, and vanishes at all other N-knots of $(\Z / N)^d$. 

It is convenient to keep in mind two following simple propositions:
\begin{proposition}\label{prop:linear_interpolation}
Suppose we have $K$ $N$-knots $\frac{\n_1}{N}, \dots, \frac{\n_K}{N}$, $\n_i \in \Z^d$ and corresponding numbers $\ell_1, \dots, \ell_K$. Then the function
\begin{align*}
g(\mathbf x) = \sum_{k=1}^K \ell_k \phi(N\mathbf x - \n_k)
\end{align*}
has the following properties:
\begin{enumerate}
\item $g(\mathbf x)$ is linear on each simplex from $\mathcal{P}_N$; 
\item $g\left( \frac{\n_k}{N} \right) = \ell_k$ for $k=1, \dots N$. For other $N$-knots $\frac{\n}{N}$, $h$ is zero: $h\left( \frac{\n}{N} \right) = 0$;
\item $g(\mathbf x)$ can be computed exactly by a network with $O(K)$ weights and $O(1)$ layers.
\end{enumerate}
\end{proposition}

\begin{proposition}\label{prop:constant_interpolation}
Suppose we have $K$ $N$-knots $\frac{\n_1}{N}, \dots, \frac{\n_K}{N}$, $\n_i \in \Z^d$ and corresponding numbers $s_1, \dots, s_K$. Suppose also that $N$-patches associated with $\frac{\n_1}{N}, \dots, \frac{\n_K}{N}$ are disjoint. Then there exists function $h(\mathbf x)$ with the following properties:
\begin{enumerate}
\item $h(\mathbf x)$ is linear on each simplex from $\mathcal{P}_N$; 
\item For $k=1, \dots N$, $h\left( \x \right) = s_k$ at an $N$-patch associated with $\frac{\n_i}{N}$; 
\item $h(\mathbf x)$ can be computed exactly by a network with $O(K)$ weights and $O(1)$ layers.
\end{enumerate}
\end{proposition}
\begin{proof}
Follows directly from \autoref{prop:linear_interpolation}. We assign value $s_k$ to all $N$-knots in $N$-patch associated with $\frac{\n_k}{N}$ and apply \autoref{prop:linear_interpolation}. Since $N$-patches of interest are disjoint, each $N$-knot has at most one assigned value.
\end{proof}

\subsection{The filtering subgrids}\label{ss:overviewandpartitioning}

Given the total number of weights $W$, we set $N = W^{1/d}$. We will assume without loss of generality that $N$ is integer. We consider triangulation $\mathcal{P}_N$ of $[0,1]^d$ on length scale $\tfrac{1}{N}$.

It is convenient to split the $N$-grid into $3^d$ disjoint subgrids with the $3\times$ grid spacing:
$$\mathbf  N_{\mathbf q} = \{\tfrac{\n}{N}: \n \in \left(\mathbf q + (3\Z)^d\right) \cap [0, N]^d\},\quad \mathbf q \in \{0, 1, 2\}^d.$$
Clearly, each subgrid contains $O(N^d)$ knots. Note that $N$-patches associated with $N$-knots in $\mathbf  N_{\mathbf q}$ are disjoint. It means, in particular, that any point $\mathbf{x} \in [0,1]^d$ lies in at most one such $N$-patch. It also means that \autoref{prop:constant_interpolation} is applicable to $\mathbf  N_{\mathbf q}$. We will use this observation in \autoref{ss:singlegrid} for constructing an efficient approximation in a neighbourhood of $\mathbf  N_{\mathbf q}$ for a single $\q$. We call the union of these $N$-patches \emph{a domain of $\mathbf  N_{\mathbf q}$}.

We compute the full approximation $\widetilde f$ as a sum
\begin{align}\label{eq:qdecomp}
\ft(\x) &= \sum_{\q \in \{0,1,2\}^d} \wt_{\q}(\x) \ft_{\q}(\x).
\end{align}

Function $\ft_{\q}(\x)$ computes $f(\x)$ with error $O(W^{-p})$ for every $\x$ in the domain of $\mathbf  N_{\mathbf q}$. For $\x$ out of the domain of $\mathbf  N_{\mathbf q}$ it computes some garbage value. We describe $\ft_{\q}(\x)$ in \autoref{ss:singlegrid}. The final approximation $\ft(\x)$ is a weighted sum of $\ft_{\q}(\x)$ with weights $\wt_{\q}(\x)$. We choose such functions $\wt_{\q}(\x)$, that $\wt_{\q}(\x)$ vanishes outside the domain of $\mathbf  N_{\mathbf q}$ and
\begin{align*}
\sum_{\q \in \{0,1,2\}^d} \wt_{\q}(\x) \equiv 1.
\end{align*}
It follows that $\ft(\x)$ is a weighted sum (with weights with the sum 1) of terms approximating $f(\x)$ with error $O(W^{-p})$. Consequently, $\ft(\x)$ approximates $f(x)$ with error $O(W^{-p})$.

Function $\wt_{\q}(\x)$ is given by applying \autoref{prop:linear_interpolation} to $N$-knots from $\mathbf  N_{\mathbf q}$ with all values $\ell_1, \ell_2, \dots, \ell_{|\mathbf  N_{\mathbf q}|}$ equals to 1. Clearly, $\wt_{\q}(\x)$ vanishes outside the domain of $\mathbf  N_{\mathbf q}$. Sum $\sum_{\q \in \{0,1,2\}^d} \wt_{\q}(\x)$ is linear on each simplex from $\mathcal{P}_N$ and equals to 1 at all $N$-knots, because each $N$-knot belongs to exactly one set $\mathbf  N_{\mathbf q}$. Consequently, this sum equals to 1 for every $\x \in [0,1]^d$. It follows from \autoref{prop:linear_interpolation} that network implementing $\wt_{\q}(\x)$ has $O(N^d) = O(W)$ weights and $O(1)$ layers.

Multiplication $\wt_{\q}(\x) \ft_{\q}(\x)$ is implemented approximately, with error $O(W^{-p})$, by network given by \cite[Proposition~3]{yarsawtooth} and requires $O(\log W)$ additional weights.

\subsection{The approximation for a subgrid}\label{ss:singlegrid}

Here we describe how we construct $\ft_{\q}(\x)$ for a single $\q \in \{0, 1, 2\}^d$. Remind that $\ft_{\q}(\x)$ computes accurate approximation for $f(\x)$ only on the domain of $\mathbf  N_{\mathbf q}$.

For any $N$-knot $\frac{\n}{N}$ in $\mathbf  N_{\mathbf q}$ we consider a cube with center at $\frac{\n}{N}$ and edge $\frac{2}{N}$:
\begin{align*}
\left\{\x \in \R^{d} : \max_{1 \leq i \leq d} \left|\x_{i} - \frac{\n_i}{N} \right| \leq \frac{1}{N} \right\}.
\end{align*}
We call such cube \emph{an $N$-cube} and denote it by $\mathbf C_{\n}$. Note that $\mathbf C_{\n} = \tfrac{\n}{N} + \mathbf C_{\mathbf{0}}$.

Remind that the domain of $\mathbf  N_{\mathbf q}$ consists of $|\mathbf  N_{\mathbf q}|$ disjoint $N$-patches associated with $N$-knots from $\mathbf  N_{\mathbf q}$. Each $\x$ from the domain of $\mathbf  N_{\mathbf q}$ belongs to exactly one such $N$-patch. We call this patch \emph{an $N$-patch for $\x$} and associated $N$-knot \emph{an $N$-knot for $\x$}. Let us denote an $N$-knot for $\x$ by $\tfrac{\n_{\q}(\x)}{N}$.

We set $M = W^{p/r}$. Note that $M^{-r} = W^{-p}$ and, therefore, we need to construct an approximation of error $O(M^{-r})$. We will assume without loss of generality that $M$ is integer and $M$ is divisible by $N$. Then $\mathcal{P}_M$ is a subpartition of $\mathcal{P}_N$. We define $M$-knot and $M$-patch similarly to $N$-knot and $N$-patch. We denote a set of all $M$-knots by $\mathbf{K}_{M}$. Note that there are $O\left((M / N)^d\right)$ $M$-knots in each $N$-patch and $N$-cube. See Fig.\ref{fig:mesh} for an illustration of all described constructions.

\begin{figure}
    \centering
    \includegraphics[scale = 0.45, clip, trim=0mm 0mm 40mm 0mm ]{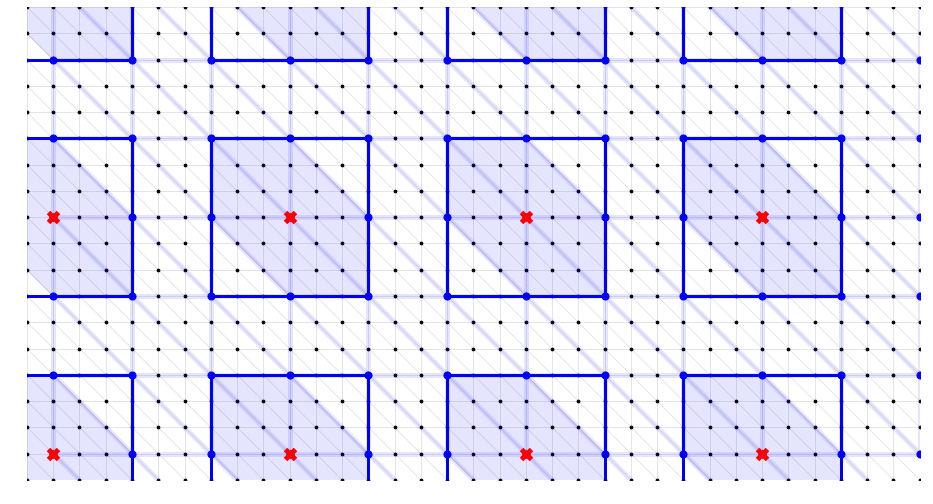}
    \caption{The partitions $\mathcal{P}_{N}$ and $\mathcal{P}_{M}$ for $d = 2$ and $\tfrac{M}{N} = 3$. The small black dots are the $M$-knots, and the thin black edges show the triangulation $\mathcal{P}_{M}$. The large blue dots are the $N$-knots; the light blue edges show the triangulation $\mathcal{P}_{N}$. The red crosses show the points of the subgrid $\mathbf{N}_{\q}$. The filled blue region is the domain of $\mathbf{N}_{\q}$. The bold blue squares show the $N$-cubes $\mathbf{C}_{\n}$ for the points of $\mathbf{N}_{\q}$.}
    \label{fig:mesh}
\end{figure}

Suppose that $\x$ lies in an $M$-patch associated with an $M$-knot $\tfrac{\m}{M}$. Consider a Taylor polynomial $P_{\m / M}(\x)$ at $\tfrac{\m}{M}$ of order $\lceil r \rceil - 1$. Standard bounds for the remainder of Taylor polynomial imply that it approximates $f(\x)$ with error $O(M^{-r})$ uniformly for $f \in F_{r,d}$. Taylor polynomial at $\tfrac{\m}{M}$ (and actually any polynomial) can be implemented with error $O(M^{-r})$ by a network with $O(\log M)$ weights and layers. We refer reader to \cite[Proposition~3]{yarsawtooth} and a proof of \cite[Theorem~1]{yarsawtooth} for details.

We can approximate $f(\x)$ with error $O(M^{-r})$ with a weighted sum of Taylor polynomials $P_{\m / M}(\x)$ at all $M$-knots:
\begin{align}\label{eq:fulltaylor}
\ft(\x) = \sum_{\tfrac{\m}{M} \in \mathbf{K}_{M}} \phi\left(M \x - \m\right) P_{\m / M}(\x).
\end{align}
Note that $\phi\left(M \x - \m\right)$ vanishes outside an $M$-patch associated with $\tfrac{\m}{M}$ and 
\begin{align*}
\sum_{\tfrac{\m}{M} \in \mathbf{K}_{M}} \phi\left(M \x - \m\right) \equiv 1.
\end{align*}

There are $M^d$ terms in \eqref{eq:fulltaylor} and calculating single term requires $O(\log M)$ weights. So, the total number of weights needed to implement \eqref{eq:fulltaylor} is $O(M^{d} \log M) = O(W^{pd / r} \log W)$. It is clearly infeasible for $p > \frac{r}{d}$. For $p = \frac{r}{d}$ it leads to approximation error $O(W^{-r/d} \log^{r/d} W)$ and makes a statement of \cite[Theorem~1]{yarsawtooth}. Note that in this construction Taylor coefficients at $M$-knots are the weights of network.

Note that terms of \eqref{eq:fulltaylor} are nonzero only for $M$-knots in an $N$-cube for $\x$. Suppose that $\x$ lies in the domain of $\mathbf{N}_{\q}$ and, therefore, has well defined $N$-knot $\tfrac{\n_{\q}(\x)}{N}$. For such $\x$ we can write
\begin{align}\label{eq:parttaylor}
\begin{split}
\ft_{\q}(\x) &= \sum_{\tfrac{\m}{M} \in \mathbf{K}_{M} \cap \mathbf{C}_{\n_{\q}(x)}} \phi\left(M \x - \m\right) P_{\m / M}(\x) \\
&= \sum_{\tfrac{\m}{M} \in \mathbf{K}_{M} \cap \mathbf{C}_{\mathbf{0}}} \phi\left(M \left(\x - \tfrac{\n_{\q}(x)}{N}\right) - \m\right) P_{\m / M + \n_{\q}(\x) / N}(\x)
\end{split}
\end{align}
There are only $(M / N)^d = W^{pd/r - 1}$ terms in \eqref{eq:parttaylor}. Therefore, if we know $\tfrac{\n_{\q}(\x)}{N}$ and Taylor coefficients for $\tfrac{\m}{M} + \tfrac{\n_{\q}(\x)}{N}$, then $\ft_{\q}(\x)$ can be implemented with error $O(M^{-r})$ by a network with $O\left((M / N)^d \log M\right) = O(W^{pd/r - 1} \log W)$ weights.

For $\x$ in the domain of $\mathbf{N}_{\q}$ it holds that $\ft_{\q}(\x) = \ft(\x)$. It follows that $\ft_{\q}(\x)$ indeed approximates $f(\x)$ with error $O(M^{-r}) = O(W^{-p})$ on the domain of $\mathbf{N}_{\q}$.

If $\x$ lies in the domain of $\mathbf{N}_{\q}$, then we can compute a single coordinate of $\tfrac{\n_{\q}(x)}{N}$ with a network given by \autoref{prop:constant_interpolation}. We need to take $\frac{\n_i}{N} \in \mathbf{N}_{\q}$ and set $s_i$ to be a corresponding coordinate of $\n_i$. We compute $\tfrac{\n_{\q}(x)}{N}$ by applying this observation to all coordinates. Constructed network has $O(|\mathbf  N_{\mathbf q}|) = O(N^d) = O(W)$ weights and $O(1)$ layers.

In \autoref{ss:coefficientscoding} we show, that (approximated) Taylor coefficients for $(M / N)^d$ $M$-knots $\tfrac{\m}{M} + \tfrac{\n}{N}$, $\tfrac{\m}{M} \in \mathbf{K}_{M} \cap \mathbf{C}_{\mathbf{0}}$ can be computed by a network with $O\left((M / N)^d\right)$ weights and layers from $c_{r, d} \leq 2(d + 1)^{\lceil r \rceil - 1}$ $\n$-dependent values. We call this values \emph{encoding weights for $\n$}.

In \autoref{ss:coefficientscoding} we describe how we construct encoding weights for a particular function $f$ and an $N$-knot $\tfrac{\n}{N}$. We show that using approximated Taylor coefficients computed from encoding weights instead of real ones leads to error bounded by $O(M^{-r}) = O(W^{-p})$. For $\x$ in the domain of $\mathbf{N}_{\q}$ we can calculate encoding weights for $\n_{\q}(\x)$ by a network given by \autoref{prop:constant_interpolation}.

Let us finalize a structure of network computing $\ft_{\q}(\x)$. For $\x$ in the domain of $\mathbf{N}_{\q}$ it
\begin{enumerate}
\item Computes $\n_{\q}(\mathbf x)$ and encoding weights for $\n_{q}(\mathbf x)$. This step is implemented by applying \autoref{prop:constant_interpolation} and requires $O(N^d) = O(W)$ weights and $O(1)$ layers;
\item Given encoding weights for $\n_{q}(x)$, computes (approximated) Taylor coefficients for all $M$-knots $\tfrac{\m}{M} + \tfrac{\n_{\q}(\x)}{N}$, $\tfrac{\m}{M} \in \mathbf{K}_{M} \cap \mathbf{C}_{\mathbf{0}}$. This step requires a network with $O\left((M / N)^d\right) = O(W^{pd/r - 1})$ weights and layers and described in \autoref{ss:coefficientscoding};
\item Given (approximated) Taylor coefficients achieved at the previous step,  computes an approximation for $P_{\m / M + \n_{\q}(\x) / N}$ for all $M$-knots $\tfrac{\m}{M} + \tfrac{\n_{\q}(\x)}{N}$, $\tfrac{\m}{M} \in \mathbf{K}_{M} \cap \mathbf{C}_{\mathbf{0}}$. The approximation with error $O(M^{-r}) = O(W^{-p})$ for a single $P_{\m / M + \n_{\q}(\x) / N}$ can be implemented by a network with $O(\log M) = O(\log W)$ weights and layers. Total number of weights needed at this step is, therefore, $O\left(|\mathbf{K}_{M} \cap \mathbf{C}_{\mathbf{0}}| \log M\right) = O\left((M / N)^d \log M\right) = O(W^{pd/r - 1} \log W)$. Computation for different $M$-knots can be done in parallel, so the total number of layers is still $O(\log W)$;
\item Given $\n_{\q}(\x)$, computed at first step, computes $\phi\left(M \left(\x - \tfrac{\n_{\q}(x)}{N}\right) - \m\right)$ for all $M$-knots $\tfrac{\m}{M} + \tfrac{\n_{\q}(\x)}{N}$, $\tfrac{\m}{M} \in \mathbf{K}_{M} \cap \mathbf{C}_{\mathbf{0}}$. It requires $O\left(|\mathbf{K}_{M} \cap \mathbf{C}_{\mathbf{0}}|\right) = O\left((M / N)^d\right) = O(W^{pd/r - 1})$ weights and $O(1)$ layers;
\item Combines outputs of steps 3 and 4 in the final approximation with \eqref{eq:parttaylor}. Multiplication with accuracy $O(M^{-r})$ can be implemented by a network with $O(\log M)$ weights and layers, so this step requires $O(|\mathbf{K}_{M} \cap \mathbf{C}_{\mathbf{0}}| \log M) = O\left((M / N)^d \log M\right) = O(W^{pd/r - 1} \log W)$ weights and $O(\log M) = O(\log W)$ layers.
\end{enumerate}

Clearly we can pass forward values achieved at early steps without increasing an asymptotic for needed number of weights and layers. 

If we sum up the total number of weights needed at each step, we obtain $O\left(W + W^{pd/r - 1} \log W\right)$. For $\tfrac{r}{d} < p < \tfrac{2r}{d}$ it is equivalent to $O(W)$ and matches the desired approximation rate. For $p = \tfrac{2r}{d}$ it is equivalent to $O(W \log W)$ and leads to the desired approximation rate up to a logarithmic factor. We show how to deal with it in \autoref{ss:logfactor}.

The total number of needed layers is $O(W^{pd/r - 1})$ and matches the desired.

\subsection{Encoding and decoding Taylor coefficients}\label{ss:coefficientscoding}

It is known that $\sim \epsilon^{-d/r}$ bits are needed to specify a function $f \in F_{r,d}$ with accuracy $\epsilon$ \cite{KolmogorovTikhomirov}. It follows from the bounds for Kolmogorov $\varepsilon$-entropy of $F_{r,d}$ derived in \cite[\S~4]{KolmogorovTikhomirov}. Here we describe how this specification can be implemented by a neural network.

First we introduce some notation. Suppose we have an $M$-knot $\tfrac{\m}{M}$. Taylor expansion $P_{\m / M}(\x)$ of $f(\x)$ at $\tfrac{\m}{M}$ is given by
\begin{align*}
P_{\m / M}(\x) &= \sum_{\k : |\k| \leq \lceil r \rceil - 1} \dfrac{D^{\k}f\left(\frac{\m}{M}\right)}{\k !} \left( \x - \dfrac{\m}{M}\right)^{\k}.
\end{align*}
We use usual convention $\k! = \prod_{i=1}^d k_i$ and $\left( \x - \tfrac{\m}{M}\right)^{\k} = \prod_{i=1}^d \left( x_i - \tfrac{m_i}{M}\right)^{k_i}$. We denote 
\begin{align*}
a_{\m, \k} = D^{\k}f\left(\frac{\m}{M}\right).
\end{align*}

We denote an approximated Taylor coefficients to be defined further in this section by $\a_{\m, \k}$. Corresponding approximated Taylor expansion is given by
\begin{align*}
\widehat{P}_{\m / M}(\x) &= \sum_{\k : |\k| \leq \lceil r \rceil - 1} \dfrac{\a_{\m, \k}}{\k !} \left( \x - \dfrac{\m}{M}\right)^{\k}.
\end{align*}

For any $\x$ in the $M$-patch associated with $\tfrac{\m}{M}$
\begin{align*}
\left| f(\x) - P_{\m / M}(\x) \right| \leq c_{r, d} M^{-r},
\end{align*}
for all $f \in F_{r, d}$ and some constant $c_{r, d}$, which does not depend on $M$ and $\m$.

We first show how we construct encoding weights associated with an $N$-knot $\tfrac{\n}{N}$. Our construction is quite similar to one from the proof of \cite[Theorem~XIV]{KolmogorovTikhomirov}, where bounds for Kolmogorov $\varepsilon$-entropy of $F_{r,d}$ were derived. Then we discuss how approximated Taylor coefficients at $M$-knots in the $N$-cube $\mathbf{C}_{\n}$ are computed from encoding weights by a network.

Our goal is to construct such approximated Taylor coefficients $\a_{\m, \k}$, that for any $\x$ in the $M$-patch associated with $\tfrac{\m}{M}$ holds $|\widehat{P}_{\m / M}(\x) - P_{\m / M}(\x)| \leq c_{r, d} M^{-r}$ for some $M$-independent constant $c_{r, d}$. The following proposition states sufficient condition on such $\a_{\m, \k}$.
\begin{proposition}\label{prop:approxbound}
Suppose that 
\begin{align}\label{eq:sufcond}
|a_{\m, \k} - \a_{\m, \k}| \leq M^{|\k| - r} \quad \forall \, \k: |\k| \leq \lceil r \rceil - 1.
\end{align}
Then for any $\x$ in an $M$-patch associated with $\tfrac{\m}{M}$
\begin{align*}
\left| \widehat{P}_{\m / M}(\x) - P_{\m / M}(\x) \right| \leq (d + 1)^{\lceil r \rceil - 1} M^{-r}.
\end{align*}
\end{proposition}
\begin{proof}
\begin{align*}
\left| \widehat{P}_{\m / M}(\x) - P_{\m / M}(\x) \right| &\leq \sum_{\k : |\k| \leq \lceil r \rceil - 1} \dfrac{1}{\k !} \left|\a_{\m, \k} - a_{\m, \k}\right| \left|\left( \x - \dfrac{\m}{M}\right)^{\k} \right| \\
&\leq \sum_{\k : |\k| \leq \lceil r \rceil - 1} M^{|\k| - r} M^{-|\k|} \\
&\leq (d + 1)^{\lceil r \rceil - 1} M^{-r}.
\end{align*}
\end{proof}

Suppose that two $M$-knots $\tfrac{\m_1}{M}$ and $\tfrac{\m_2}{M}$ are adjacent and we have $\a_{\m_1, \kh}$, $|\kh| \leq \lceil r \rceil - 1$ satisfying \eqref{eq:sufcond}. Another convenient proposition we use further shows how to construct an accurate approximation for Taylor coefficients at $\tfrac{\m_2}{M}$.

\begin{proposition}\label{prop:adjacent}
Suppose that two $M$-knots $\tfrac{\m_1}{M}$ and $\tfrac{\m_2}{M}$ are adjacent. Suppose that approximated Taylor coefficients $\a_{\m_1, \kh}$, $|\kh| \leq \lceil r \rceil - 1$ at $\tfrac{\m_1}{M}$ satisfy \eqref{eq:sufcond}. Then we can find such $c_{\k, \kh}$ and $\at_{m_2, \k}$, $|\k|, |\kh| \leq \lceil r \rceil - 1$, that
\begin{enumerate}
    \item For all $\k: |\k| \leq \lceil r \rceil - 1$
        \begin{align*}
            \at_{m_2, \k} = \sum_{\kh: |\kh| \leq \lceil r \rceil - 1} c_{\k, \kh} \cdot \widehat a_{\m_1, \kh};
        \end{align*}
    \item For all $\k: |\k| \leq \lceil r \rceil - 1$ 
        \begin{align}\label{eq:adjapproxbound}
            |a_{\m_2, \k} - \at_{\m_2, \k}| < 4 M^{|\k| - r};
        \end{align}
    \item Coefficients $c_{\k, \kh}$ depend only on the relative position of $\tfrac{\m_1}{M}$ and $\tfrac{\m_2}{M}$.
\end{enumerate}
\end{proposition}
\begin{proof}
Remind that $M$-knots $\tfrac{\m_1}{M}$ and $\tfrac{\m_2}{M}$ are adjacent. Let us consider first component of $\m_1$ and $\m_2$ independently and assume without loss of generality that $\m_1 = (m_1, \overline{\m})$ and $\m_2 = (m_1 + 1, \overline{\m})$.

Standard bounds for a remainder of Taylor series partial sum imply, that for any $\k = (k_1, \dots, k_d)$ and $f \in F_{r, d}$
\begin{align*}
\left|D^{(k_1, \dots, k_d)}f\left(\dfrac{\m_2}{M}\right) - \sum_{n = 0}^{\lceil r \rceil - 1 - |\k|} \dfrac{D^{(k_1 + n, \dots, k_d)}f\left(\dfrac{\m_1}{M}\right)}{n!} \cdot \dfrac{1}{M^n} \right| \leq M^{|\k| - r}.
\end{align*}
In our notation
\begin{align}\label{eq:adjtaylorbound}
\left|a_{\m_2, (k_1, \dots, k_d)} - \sum_{n = 0}^{\lceil r \rceil - 1 - |\k|} \dfrac{a_{\m_1, (k_1 + n, \dots, k_d)}}{n!} \cdot \dfrac{1}{M^n} \right| \leq M^{|\k| - r}.
\end{align}
From the proposition that coefficients $\a_{\m_1, \k}$ satisfy \eqref{eq:sufcond} it follows that
\begin{align}\label{eq:substbound}
\begin{split}
\left|\sum_{n = 0}^{\lceil r \rceil - 1 - |\k|} \dfrac{\left(a_{\m_1, (k_1 + n, \dots, k_d)} - \a_{\m_1, (k_1 + n, \dots, k_d)}\right)}{n!} \cdot \dfrac{1}{M^n} \right| &\leq \sum_{n = 0}^{\lceil r \rceil - 1 - |\k|} \dfrac{M^{|\k| + n - r}}{n!} \cdot \dfrac{1}{M^n} \\
&= M^{|\k| - r} \sum_{n = 0}^{\lceil r \rceil - 1 - |\k|} \dfrac{1}{n!} \\
&< e M^{|\k| - r} < 3 M^{|\k| - r}.
\end{split}
\end{align}
Combining \eqref{eq:adjtaylorbound} and \eqref{eq:substbound} we obtain
\begin{align*}
\begin{split}
\left|a_{\m_2, (k_1, \dots, k_d)} - \sum_{n = 0}^{\lceil r \rceil - 1 - |\k|} \dfrac{\a_{\m_1, (k_1 + n, \dots, k_d)}}{n!} \cdot \dfrac{1}{M^n} \right| < 4 M^{|\k| - r}.
\end{split}
\end{align*}
It follows that if for each $\k = (k_1, \dots, k_d)$ we set
\begin{align}\label{eq:adjtaylorrepr}
    \at_{\m_2, (k_1, \dots, k_d)} = \sum_{n = 0}^{\lceil r \rceil - 1 - |\k|} \dfrac{\a_{\m_1, (k_1 + n, \dots, k_d)}}{n!} \cdot \dfrac{1}{M^n},
\end{align}
then $\at_{\m_2, \k}$ satisfy \eqref{eq:adjapproxbound}. It remains to note that coefficients in \eqref{eq:adjtaylorrepr} depend only on the relative position of $\tfrac{\m_1}{M}$ and $\tfrac{\m_2}{M}$, but not on $f \in F_{r, d}$, values $\a_{\m_1, \k}$ or $M$-knots $\tfrac{\m_1}{M}$ and $\tfrac{\m_2}{M}$ themselves.
\end{proof}

Now we are ready to describe how we find $\a_{\m, \k}$ for all $M$-knots $\tfrac{\m}{M}$ from a given $N$-cube $\mathbf{C}_{\n}$. We enumerate $M$-knots lying in $\mathbf{C}_{\n}$ with numbers $t = 1, \dots, (2 M / N + 1)^d$ and denote them $\tfrac{\m_{\n, t}}{M}$. We inductively construct $\a_{\m_{\n, t}, \k}$ satisfying \eqref{eq:sufcond} for all $M$-knots $\tfrac{\m_{\n, t}}{M}$. We choose such an enumeration, that two consequent $M$-knots are adjacent. 

We set $\a_{\m_{\n, 1}, \k} = a_{\m_{\n, 1}, \k}$. Such $\a_{\m_{\n, 1}, \k}$ clearly satisfy \eqref{eq:sufcond}. Suppose that we have constructed $\a_{\m_{\n, t}, \k}$ satisfying \eqref{eq:sufcond}. Since $M$-knots $\tfrac{\m_{\n, t}}{M}$ and $\tfrac{\m_{\n, t+1}}{M}$ are adjacent, we can apply \autoref{prop:adjacent} to get $\at_{m_{\n, t+1}, \k}$, $|\k| \leq \lceil r \rceil - 1$ satisfying \eqref{eq:adjapproxbound}. It follows that there exist such integers $B_{\n, \k, t}$, that $|B_{\n, \k, t}| \leq 3$ and
\begin{align*}
\left|a_{\m_{\n, t+1}, \k} - \at_{\m_{\n, t+1}, \k} - M^{|\k| - r} B_{\n, \k, t} \right| \leq M^{|\k| - r}.
\end{align*}
We set
\begin{align}\label{eq:coefformula}
\a_{\m_{\n, t+1}, \k} = \at_{\m_{\n, t+1}, \k} + M^{|\k| - r} B_{\n, \k, t}.
\end{align}
Then coefficients $\a_{\m_{\n, t+1}, \k}$ satisfy \eqref{eq:sufcond} as desired. See Fig.\ref{fig:encoding} for an illustration of algorithm of determining $\a_{\m_{\n, t+1}, \k}$.

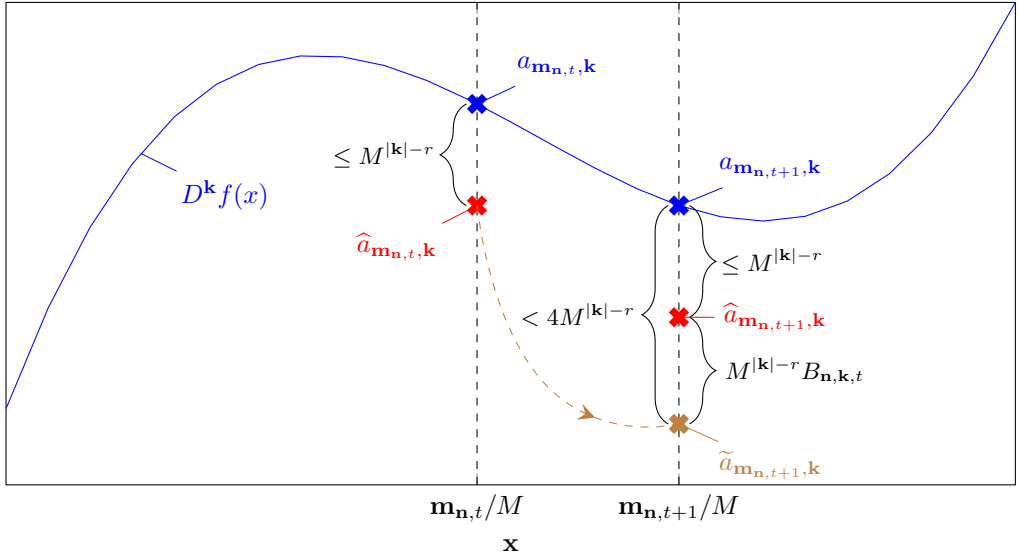
\begin{figure}
\centering
\input{encoding.tikz}
\caption{An illustration of determining approximated Taylor coefficients at $\tfrac{\m_{\n, t+1}}{M}$ from known approximated Taylor coefficients at $\tfrac{\m_{\n, t}}{M}$. The blue line is $D^{\k} f(x)$ and the blue crosses are its values $a_{\m_{\n, t}, \k}$ and $a_{\m_{\n, t+1}, \k}$ at $M$-knots $\tfrac{\m_{\n, t}}{M}$ and $\tfrac{\m_{\n, t+1}}{M}$ respectively. Red crosses are desired approximations $\a_{\m_{\n, t}, \k}$ and $\a_{\m_{\n, t+1}, \k}$ for $a_{\m_{\n, t}, \k}$ and $a_{\m_{\n, t+1}, \k}$ satisfying \eqref{eq:sufcond}. Given $\a_{\m_{\n, t}, \k}$, we first apply \autoref{prop:adjacent} to get $\at_{\m_{\n, t+1}, \k}$ satisfying \eqref{eq:adjapproxbound}. This step is illustrated by the brown dashed arrow and brown cross is $\at_{\m_{\n, t+1}, \k}$. Then we choose such $B_{\n, \k, t} \in \{-3, \dots, 3\}$, that $\a_{\m_{\n, t+1}, \k} = \at_{\m_{\n, t+1}, \k} + M^{|\k| - r} B_{\n, \k, t}$ satisfy \eqref{eq:sufcond}.}
\label{fig:encoding}
\end{figure}

For a single $\k$ we encode $(2M / N + 1)^d$ values $B_{\n, \k, t}$ by a single base-7 number $b_{\n, \k}$
\begin{align*}
b_{\n, \k} = \sum_{t=1}^{(2M / N + 1)^d} 7^{-t} \left(B_{\n, \k, t} + 3\right)
\end{align*}
Numbers $b_{\n, \k}$ and $\a_{\m_{\n, 1}, \k} = a_{\m_{\n, 1}, \k}$ are encoding weights for $\n$. There are $c_{r, d} \leq 2(d + 1)^{\lceil r \rceil - 1}$ encoding weights.

Now we describe how a network reconstruct all $\a_{\m_{\n, t}, \k}$ from encoding weights. Numbers $B_{\n, \k, t}$ can be reconstructed from $b_{\n, \k}$ by a ReLU network with $O\left((M/N)^d\right)$ weights and layers. We refer to \cite[5.2.2]{yaropt}, where similar reconstruction is described for ternary numbers. Given $\a_{\m_{\n, t}, \k}$ and $B_{\n, \k, t}$, we first compute $\at_{\m_{\n, t+1}, \k}$ with \eqref{eq:adjtaylorrepr} and then we compute $\a_{\m_{\n, t+1}, \k}$ with \eqref{eq:coefformula}. We need $O(1)$ weights and layers at each step, so
the total number of needed weights and layers is $O\left((M/N)^d\right)$.

For given $\x \in [0, 1]^d$ and $\q \in \{0, 1, 2\}^d$ we obtain encoding weights for $\n_{\q}(\x)$ by applying \autoref{prop:constant_interpolation}. Note that \autoref{prop:adjacent} implies that coefficients in \eqref{eq:adjtaylorrepr} depend only on the relative position of $M$-knots $\tfrac{\m_{\n, t}}{M}$ and $\tfrac{\m_{\n, t+1}}{M}$. It follows that if we choose similar enumeration of $M$-knots for all $N$-cubes $\mathbf{C}_{\n}$, $\tfrac{\n}{N} \in \mathbf  N_{\mathbf q}$, then we can use a network described in previous paragraph for all possible values of $\n_{\q}(\x)$.

Note that encoding weights $b_{\n, \k}$ can be represented as $\sim(M/N)^d$-bits numbers while encoding weights $\a_{\m_{\n, 1}, \k} = a_{\m_{\n, 1}, \k}$ can be arbitrary real numbers. Remind that described construction requires $|\a_{\m_{\n, 1}, \k} - a_{\m_{\n, 1}, \k}| \sim M^{|\k| - r}$. It follows that if we want to encode $\a_{\m_{\n, 1}, \k}$ by a finite number of bits as well, then we need $\sim \log M$ additional bits to achieve desired accuracy. 

\subsection{Getting rid of logarithmic factor}\label{ss:logfactor}

Remind that logarithmic factor arises in the construction described in \ref{ss:singlegrid} in case $p = \tfrac{2r}{d}$. This is because we construct $\ft_{\q}(\x)$ in form \eqref{eq:parttaylor} with $O\left(W\right)$ terms and we need $O(\log W)$ weights to implement an approximated Taylor sum arising in each term.

Note that for a particular $\x$ most terms in \eqref{eq:parttaylor} vanishes since $\phi(M (\x - \tfrac{\n_{\q}(x)}{N}) - \m) = 0$ and there is no need to compute $P_{\m / M + \n_{\q}(\x) / N}(\x)$ for such terms. If we perform Taylor sum calculation for only a constant number of non-vanishing terms, then the total number of needed weights reduces to $O(W + \log W)$. We can apply technique used in \ref{ss:overviewandpartitioning} for detecting nonvanishing terms from input $\x$.

We split all $M$-knots lying in $N$-cube $\mathbf{C}_{\mathbf{0}}$ into a disjoint union of $3^{d}$ sets
\begin{align*}
\mathbf{M}_{\mathbf{s}} = \{\tfrac{\m}{M}: \m \in \left(\mathbf s + (3\Z)^d\right) \cap \mathbf{C}_{\mathbf{0}}\},\quad \mathbf{s} \in \{0, 1, 2\}^d.
\end{align*}
$M$-patches associated with $M$-knots in $\mathbf{M}_{\mathbf{s}}$ are disjoint. We call their union the domain of $\mathbf{M}_{\mathbf{s}}$. If $\x - \tfrac{\n_{\q}(\x)}{N}$ lies in the domain of $\mathbf{M}_{\mathbf{s}}$, there is exactly one such $\tfrac{\m_{\q, \mathbf{s}}(\x)}{M}$, that $\x - \tfrac{\n_{\q}(\x)}{N}$ lies in the $M$-patch associated with $\tfrac{\m_{\q, \mathbf{s}}}{M}$. 
We can rewrite \eqref{eq:parttaylor} as
\begin{align}\label{eq:nologform}
\ft_{\q}(\x) &= \sum_{\mathbf{s} \in \{0, 1, 2\}^d} \left[\ft_{\q, \mathbf{s}}(\x) \sum_{\tfrac{\m}{M} \in \tfrac{\n_{\q}(\x)}{N} + \mathbf{M}_{\mathbf{s}}} \phi\left(M \left(\x - \tfrac{\n_{\q}(x)}{N}\right) - \m\right) \right].
\end{align}
Here $\ft_{\q, \mathbf{s}}(\x)$ is a function, which calculates $P_{\m_{\q, \mathbf{s}}(\x) / M + \n_{\q} / N}(\x)$ if $\x - \tfrac{\n_{\q}(\x)}{N}$ lies in the domain of $\mathbf{M}_{\mathbf{s}}$, and some garbage value otherwise. We also require that $\ft_{\q, \mathbf{s}}(\x)$ computes an approximation for a Taylor series partial sum only once. The total number of partial sums computed by network implementing $\ft_{\q}(\x)$ in form \eqref{eq:nologform} is therefore reduced to $3^{d}$. The total number of weights needed to implement $\ft_{\q}(\x)$ reduces from $O(W \log W)$ to $O(W)$.

To compute such $\ft_{\q, \mathbf{s}}(\x)$ we only need to determine approximated Taylor coefficients for $\tfrac{\m_{\q, \mathbf{s}}(\x)}{M} + \tfrac{\n_{\q}(\x)}{N}$ among all coefficients. For each $\tfrac{\m}{M} \in \mathbf{M}_{\mathbf{s}}$ we construct function $\widehat{w}_{\mathbf{s}, \m}(\x)$, which equals to 1 in the $M$-patch associated with $\tfrac{\m}{M}$ and vanishes in other patches of the domain of $\mathbf{M}_{\mathbf{s}}$. Knowing values $\widehat{w}_{\mathbf{s}, \m}(\x - \tfrac{\n_{\q}(\x)}{N})$ we clearly can get Taylor coefficients for $\tfrac{\m_{\q, \mathbf{s}}(\x)}{M} + \tfrac{\n_{\q}(\x)}{N}$ from all Taylor coefficients computed by network. 

\begin{remark}\label{rm:nolog}
Similar reasoning can be applied to the case $p = \tfrac{r}{d}$. In this case we do not consider an $M$-grid at all, but we still can split $N$-grid into $3^d$ disjoint sets and compute approximated Taylor sum once for each set. In this case weight assignment map is continuous and even linear on $f$.
\end{remark}

\section{Theorem \ref{th:constwidth}: proof details}\label{sec:proofconstwidth}
We follow the network construction used in the proof of Theorem \ref{th:deepphase} and described in Subsections \ref{ss:overviewandpartitioning},\ref{ss:singlegrid}. We want to show that this construction can be realized within a ReLU network of width $2d+10.$ As explained in Section \ref{sec:adaptivity}, we slightly modify the construction, so that we don't update the Taylor coefficients at new $M$-patches, but rather compute them afresh. This will give a slight increase in the size of the network. Accordingly, we define parameters $N,M$ in terms of the required accuracy $\epsilon$ rather than the number of weights: specifically, we set $M=\epsilon^{-1/r}$ and $N=\epsilon^{-1/(2r)}.$

Following \cite{yaropt}, we think of the width-$(2d+10)$ network as $2d+10$ ``channels'' that are interconnected and can exchange information. We reserve $d$ channels for passing forward the scalar components of the input vector $\mathbf x$ and one channel for accumulating the approximation $\widetilde{f}(\mathbf x)$. The other channels are used for intermediate computations.

The first step in computing the approximation $\widetilde f(\mathbf x)$ is the finite decomposition \ref{eq:qdecomp} of $\widetilde f$ over $\mathbf q$-subgrids.  The decomposition can be implemented in the width-$(2d+10)$ network in the serial fashion, so we only need to consider computation of a single term $\wt_{\q}(\x) \ft_{\q}(\x)$. 

The weight $\wt_{\q}(\x)$ is just a linear combination of $O(N^d)$ functions $\phi(N\mathbf x-\mathbf n),$ and $\phi$ can be computed by a constant-size chain of linear and ReLU operations (see \cite[Section~4.2]{yaropt}). Thus, $\wt_{\q}(\x)$ can be computed by a subnetwork using just 2 channels and depth $O(\epsilon^{-d/(2r)}).$ On the other hand, we will show below that $\ft_{\q}(\x)$ can be computed by a subnetwork using $d+8$ channels and depth $O(\epsilon^{-d/(2r)}\log(1/\epsilon)).$ We can then pass the values $\wt_{\q}(\x)$ and $\ft_{\q}(\x)$ to the third subnetwork computing an $O(\epsilon)$-approximation to the product $\wt_{\q}(\x) \ft_{\q}(\x)$. This approximate product can be computed by a width-4 subnetwork of depth $O(\log(1/\epsilon))$ (see \cite[Proposition~3]{yarsawtooth}). Thus the total computation of the term $\wt_{\q}(\x) \ft_{\q}(\x)$, and hence of the whole approximation $\widetilde f(\mathbf x)$ can be done with necessary accuracy $\epsilon$ within the width-$(2d+10)$ network of depth $L=O(\epsilon^{-d/(2r)}\log(1/\epsilon)).$ By inverting this relation, we get $\epsilon = O(L^{-2r/d}\log^{2r/d}L),$ as desired.

We return now to the computation of $\ft_{\q}(\x)$. It is based on the expansion \eqref{eq:parttaylor} and can be performed as described later in that subsection. We examine now indivudual steps and how they can be implemented in our fixed-depth network.
\begin{enumerate}
    \item The $N$-knot positions $\mathbf n_{\mathbf q}(\mathbf x)$ associated with $\mathbf x$ are computed using a linear combination of $O((M/N)^d)$ functions of the form $\phi(N\mathbf x-\mathbf n_k)$. This computation can be performed in a subnetwork of width 2 and depth $O(\epsilon^{-d/2r})$. We reserve $d$ channels to pass forward the scalar components of $\mathbf n_{\mathbf q}(\mathbf x)$. Additionally, we reserve one channel for passing forward the encoding weight corresponding to this $\mathbf n_{\mathbf q}(\mathbf x)$. The encoding weight gets transformed as it passes along the network and bits get decoded from it. Additional 3 channels are sufficient for bit decoding (see \cite{yaropt} for a description of the decoding procedure). 
    \item We traverse the $O((M/N)^d)$ $M$-knots of the $N$-patch corresponding to $\mathbf n_{\mathbf q}$ and decode from the encoding weight the Taylor coefficients of degree up $\lceil r\rceil-1$ at these knots. It is sufficient to know these coefficients with precision $O(\epsilon^r),$ so each Taylor coefficient can be encoded by $K_{\max}=O(\log(1/\epsilon))$ bits $\{b_k\}_{k=0}^{K_{\max}}$, and reconstructed by accumulating the linear combination $\sum_{k=0}^{K_{\max}} 2^{-k}b_k$. Thus, the total required number of bits in the encoding weight is $O(\epsilon^{-d/(2r)})\log(1/\epsilon)$. Also, all the necessary coefficients can be reconstructed using $O(\epsilon^{-d/(2r)})\log(1/\epsilon)$ layers of width 4.
    \item At each $M$-knot $\mathbf m/M+\mathbf n_{\mathbf q}(\mathbf x)/N$ in the $N$-patch, we compute the respective Taylor polynomial $P_{\mathbf m/M+\mathbf n_{\mathbf q}(\mathbf x)/N}(\mathbf x)=\sum_{\mathbf k:|\mathbf k|\le \lceil r\rceil-1}a_{\mathbf k}(\mathbf x-(\mathbf m/M+\mathbf n_{\mathbf q}(\mathbf x)/N))^{\mathbf k}$. The values of $\mathbf x$ and $\mathbf n_{\mathbf q}(\mathbf x)$ are provided from the reserved channels, and $\mathbf m$ is defined in the network weights. We don't need to know all the coefficients at once, since the polynomial can be computed serially, one monomial after another, and one multiplication after another. To ensure accuracy $\epsilon$, each multiplication requires depth $O(\log(1/\epsilon))$ and width $4$. The total polynomial can then be accumulated using a subnetwork of depth $O(\log(1/\epsilon))$ and width $5$. 
    \item Computation of the values $\phi\big(M (\x - \tfrac{\n_{\q}(x)}{N}) - \m\big)$ can be performed in 2 channels using $O(\epsilon^{-d/(2r)})$ layers in total.
    \item Once the factors are computed, each product  $\phi\big(M (\x - \tfrac{\n_{\q}(x)}{N}) - \m\big) P_{\m / M + \n_{\q}(\x) / N}(\x)$ can be computed with accuracy $O(\epsilon)$ in a subnetwork of width $4$ and with $O(\log(1/\epsilon))$ layers, which gives $O(\epsilon^{-d/(2r)}\log(1/\epsilon))$ layers in total.
\end{enumerate}
Summarizing, we see that the computation of $\widetilde f_{\mathbf q}(\mathbf x)$ can be implemented with accuracy $O(\epsilon)$ in a subnetwork occupying $d+8$ channels and spanning $O(\epsilon^{-d/(2r)}\log(1/\epsilon))$ layers, as claimed. 

\section{Theorem \ref{th:deepsecond}: proof}\label{sec:deepsecondproof}
We generally follow the proof of Theorem \ref{th:deepphase} given in Sections 
\ref{sec:phasediag}
and \ref{sec:proofdeepphase}, and adapt it to the new setting. We start by reducing the approximation by $\sigma$-networks to deep polynomial approximations. We show that the ReLU activation function can be efficiently approximated by iterated polynomials, which allows us to reproduce some parts of the proof of Theorem \ref{th:deepphase} simply by approximating the ReLU. However, other parts, in particular the decoder subnetwork and the selection of the encoding weight, will require more significant changes. 

\paragraph{Step 1: Reduction to polynomial approximation.}
 
\begin{lemma}
Suppose that the activation function $\sigma$ has a point $x_0$ where the second derivative $\tfrac{d^2\sigma}{dx^2}(x_0)$ exists and is nonzero. Then, for any multivariate polynomial $u$, there exists a network architecture such that the polynomial $u$ can be approximated with any accuracy on any bounded set by a $\sigma$-network with this architecture by suitably assigning the weights.    
\end{lemma}
\begin{proof}
For $u(x)=x^2$, the desired $\sigma$-network is $$\widetilde u_\delta(x)=(\tfrac{d^2\sigma}{dx^2}(x_0))^{-1}\tfrac{1}{\delta^2}\big(\sigma(x_0+x\delta)+\sigma(x_0-x\delta)-2\sigma(x_0)\big)$$ with a small $\delta$. For any other polynomial, the network can be constructed by using $\widetilde u_\delta$, the polarization identity $xy=\tfrac{1}{2}((x+y)^2-(x-y)^2),$ and linear operations.  
\end{proof}
In view of this lemma, in the sequel we will treat $\sigma$-networks as if capable of exactly implementing any polynomial using some finite architecture. Also, we note that under our assumption on the activation functions and in contrast to ReLU-networks, multiplications can be implemented with any accuracy by fixed-size subnetworks.

\paragraph{Step 2: Fast polynomial approximation of thresholds and ReLU.}
Consider the polynomial $u(x)=\tfrac{1}{2}x(3-x^2)$, which in particular has the following properties:
\begin{enumerate}
\item $u(0)=0$ and $u(\pm 1)=\pm 1$;
\item $u$ is monotone increasing on $[-1,1]$;
\item $\tfrac{du}{dx}(\pm 1)=0$.
\end{enumerate}
Let $u_n$ be the $n$'th iterate of $u$:
\begin{equation}\label{eq:un}u_n=\underbrace{u\circ\ldots \circ u}_{n}.
\end{equation}
\begin{lemma}\label{th:un}\hfill 
\begin{enumerate}
\item $|u_n(x)-\operatorname{sgn}(x)|\le |\operatorname{sgn}(x)-x|^{2^{n/2}}$ for any $x\in[-1,1].$
\item $|xu_n(x)-|x||\le 2^{-n/2}$ for any $x\in[-1,1].$
\end{enumerate}
\end{lemma}
\begin{proof}
1. Make the change of variables $x=v(y)=1-y$ and let $\widetilde u=v^{-1}\circ u\circ v$. Then $\widetilde u(y)=\tfrac{1}{2}y^2(3-y).$ It is easy to check that $\widetilde u(y)\le y^{\sqrt{2}}$ for any $y\in[0,1].$ Since both $\widetilde u(y)$ and $y^{\sqrt{2}}$ are monotone increasing on $[0,1],$ there is a similar inequality for their $n$'th iterates: $\widetilde u_n(y)\le y^{2^{n/2}}.$ This gives the desired bound for $x\in[0,1]$. The bound for $x\in[-1,0]$ follows by symmetry. 

2. By Statement 1, $|xu_n(x)-|x||\le |x||\operatorname{sgn}(x)-x|^{2^{n/2}}\le 2^{-n/2}$ for $x\in[-1,1].$
\end{proof}
The lemma implies, in particular, that a size-$O(n)$ $\sigma$-network can provide an approximation of accuracy $2^{-n/2}$ for the functions $|x|$ and $x_+$ on the segment $[-1,1]$ . The ReLU $x_+$ is approximated by $\tfrac{1}{2}(xu_n(x)+x)$.

\paragraph{Step 3: Reduction to $\ft_{\q}$.} In the original proof of faster rates for ReLU given in Section \ref{sec:proofdeepphase}, the first step was to represent the approximation $\widetilde f$ by a finite expansion \eqref{eq:qdecomp} over $3^d$ subgrids indexed by $\mathbf q \in \{0,1,2\}^d$: 
\begin{align*}
\ft(\x) &= \sum_{\q \in \{0,1,2\}^d} \wt_{\q}(\x) \ft_{\q}(\x).
\end{align*}
In the original proof, the ``filtering functions'' $\wt_{\q}$ were linear combinations of $O(N^d)$ shifted and rescaled piecewise linear ``spike'' functions $\phi$ (see Prop. \ref{prop:linear_interpolation}). The function $\phi$ can be constructed using several linear and ReLU operations (see \cite[Section~4.2]{yaropt}).

We observe now that, using Lemma \ref{th:un}, we can very efficiently approximate the spike functions $\phi$ and then the full filtering functions $\wt_{\q}$ by polynomials, simply by approximating each ReLU by the polynomial $\tfrac{1}{2}(xu_n(x)+x)$. Indeed, such an approximation of $\phi$ has accuracy $O(2^{-n/2})$ for a size-$O(n)$ $\sigma$-network (propagation of the error in the computation can be controlled in the standard way, using the Lipshitz continuity of ReLU). We need to remember, however, that the approximation in Lemma \ref{th:un} is valid only on the segment $[-1,1],$ while the shifted and rescaled spike $\phi(N\mathbf x-\mathbf n)$ requires ReLUs to act on a domain of size $O(N)$ if $\mathbf x\in[0,1]^d.$ The domain adaptation can be achieved simply by rescaling the ReLU using the identity $(Nx)_+=Nx_+$. As a result, by adding approximations for all the spikes, we can approximate $\wt_{\q}$ by a $\sigma$-network of size $O(nN^d)$ and depth $O(n)$ with uniform accuracy $O(2^{-n/2}N^{d+1})$. Let $N=W^{(1-\delta)/d}$ with some small $\delta>0$, and $n=c\log_2 W$ with some $c>2\tfrac{(1-\delta)(d+1)+r}{d}$. Then the accuracy of the $\wt_{\q}$ network is within the desired bound $\epsilon=O(W^{-r/d})$, while the size of the $\wt_{\q}$ network is $O(W^{1-\delta}\log W)$, also within the desired bound $W$.  

This shows that our task is essentially reduced to implementing the functions $\ft_{\q}$. We examine now the ReLU implementation of $\ft_{\q}$ summarized into 5 steps in the end of Section \ref{ss:singlegrid}. We observe that steps 3-5 (computation of the weighted sum of Taylor approximations in a $N$-patch) can be easily implemented with the activation function $\sigma$ instead of ReLU, by invoking again Lemma \ref{th:un} where necessary. In contrast, steps 1-2 require  more significant modifications since they directly involve encoding weights that need to be handled with a  high precision ($\sim W^{pd/r-1}$ bits). We first describe the suitable modification of step 2 (bit extraction), and then of step 1 (finding $\mathbf n_{\mathbf q}(\mathbf x)$ and the respective encoding weight).

\paragraph{Step 4: Bit extraction.}
The standard bit extraction procedure (see \cite{bartlett1998almost} and Fig. \ref{fig:approxdiscrete}) decodes a binary sequence from the encoding weight using threshold activation functions ($\lfloor\cdot\rfloor$) or their approximations by ReLU. In our present setting of polynomial approximation, we use instead a polynomial dinamical system. Specifically, consider the polynomial $v(x)=2-3x^2$. Consider the disjoint intervals $I_0=[\tfrac{1}{2},1], I_1=[-1,-\tfrac{1}{2}]$ and observe that they are contained in the interval $[-1,1]$ which is in turn contained in either of the images $v(I_0), v(I_1)$. Consider a sequence $w_1,\ldots,w_n$ defined by $w_k=v(w_{k-1})$ with some initial value $w_1.$   
\begin{lemma}\label{lm:binarypol}
For any binary sequence $b_1,\ldots,b_n\in \{0,1\}$, there exists an interval $I\subset [-1,1]$ of length at least $6^{-n}$ such that for any initial value $w_1\in I$ we have $w_k\in I_{b_k}$ for all $k=1,\ldots,n$.   
\end{lemma}
\begin{proof}
The interval $I$ can be constructed by sequentially forming pre-images, $I^{(k-1)}=v^{-1}(I^{(k)})\cap I_{b_{k-1}}$, where $ k=n,n-1,\ldots,2,$ and $I^{(n)}=I_{b_{n}}.$ Then $I=I^{(1)};$ the lower bound on the length of $I$ follows since $|\tfrac{dv}{dx}|\le 6$ on $[-1,1].$ 
\end{proof}
The lemma shows that we can decode a length-$n$ binary sequence by a $\sigma$-network of size $O(n)$ starting from an encoding weight defined with precision $6^{-n}$. In contrast to the original ($\lfloor\cdot\rfloor$-based) bit extraction, the values decoded in the present polynomial procedure contain some uncertainty: we only know that $w_k$ belong to one of the intervals $I_0$ or $I_1$. However, this uncertainty is not important: first, we can reduce it to an arbitrary magnitude by small-size subnetworks implementing a polynomial $u_n$ from Lemma \ref{th:un}; second, by Proposition \ref{prop:adjacent} and Eq.\eqref{eq:sufcond}, some level of uncertainty in the Taylor coefficients $\a_{\m, \k}$ is tolerable.

\paragraph{Step 5: Computation of the encoding weight corresponding to given input $\mathbf x$.}
In the proof for ReLU networks, the position $\mathbf n_{\mathbf q}(\mathbf x)$ of the $N$-knot containing the given point $\mathbf x$, and the respective encoding weight, were determined   {exactly} thanks to the ability of ReLU networks to exactly represent functions piecewise linear on the standard triangulation (see Proposition \ref{prop:constant_interpolation}). This is no longer possible with general activation functions $\sigma$ or polynomials; any $\sigma$-network trying to determine the encoding weight will inevitably do it with some error. However, though the precision requirement for encoding weights is high,  we can use part 1 of Lemma \ref{th:un} to bring this error to an acceptable level without substantially increasing the network size. 

Indeed, consider a particular $N$-knot $\tfrac{\mathbf n}{N}\in\mathbf N_{\mathbf q}$ and first construct a map $z_{\mathbf n}(\mathbf x)$ such that $z_{\mathbf n}(\mathbf x)\in [\tfrac{1}{2},1]$ for $\mathbf x$ belonging to the corresponding $N$-patch, while $z_{\mathbf n}(\mathbf x)\in [-1,-\tfrac{1}{2}]$ for $\mathbf x$ belonging to the other $N$-patches. Arguing as in Step 3, such a map can be implemented by a $\sigma$-network of size $O(\log W)$, by approximating the respective ReLU map. 

Next, let $z_{\mathbf n, n}=u_n\circ z_{\mathbf n},$ where $u_n$ is given in Eq.\eqref{eq:un}. Using Lemma \ref{th:un} with $n=O(\log D)$, we can ensure that $|z_{\mathbf n, n}(\mathbf x)-1|<7^{-D}$ on the $\mathbf n$'th patch while $|z_{\mathbf n, n}(\mathbf x)+1|<7^{-D}$ on the other patches. Here, $D$ corresponds to the number of iterations in Lemma \ref{lm:binarypol} and is proportional to the depth of the decoding subnetwork, i.e. $D\sim (M/N)^d\sim W^{pd/r-1}$ so that $\log D = O(\log W).$

We can now combine all the maps $z_{\mathbf n, n}$ into the map $Z(\mathbf x)=\tfrac{1}{2}\sum_{\mathbf n:\mathbf n/N\in\mathbf N_{\mathbf q}}(z_{\mathbf n, n}(\mathbf x)+1)w_{\mathbf n},$ where $w_{\mathbf n}$ is the desired encoding weight in the $\mathbf n$'th patch. By construction, for any $\mathbf x$ in the $\mathbf n$'th patch we have $|Z(\mathbf x)-w_{\mathbf n}|=O(N^d 7^{-D})$, which satisfies the accuracy requirement $6^{-D}$ of Lemma \ref{lm:binarypol}. On the other hand, the size of the $\sigma$-network implementing $Z(\mathbf x)$ is $O(N^d\log W).$ Choosing $N\sim W^{(1-\delta)/d}$ with arbitrarily small $\delta>0,$ this size fits the available budget $W.$

\section{Expressiveness of networks with Lipschitz activation functions and slowly growing weights}\label{sec:weightbounds}
In this section we clarify why, as mentioned in Section \ref{sec:otheractiv}, under mild assumptions on the growth of network weights, networks with any bounded Lipschitz activation function (in particular, the standard sigmoid $\sigma(x)=1/(1+e^{-x})$) can only achieve the approximation rates $p\le \tfrac{2r}{d}.$ This follows from existing upper bounds on the covering numbers for such networks, in particular \cite[Theorem 14.5]{anthony2009neural}. 

Specifically, consider a neural network with the following properties. Suppose that the network neurons have (possibly different) Lipschitz activation functions $\sigma$ such that $|\sigma(x)|\le b$ and $|\sigma(x)-\sigma(y)|\le a|x-y|$ for all $x,y\in\mathbb R$. Suppose that there is a constant  $V>1/a$ such that for any weight vector $\mathbf w$ associated with a particular neuron, its $l^1$-norm $\|\mathbf w\|_1$ is bounded by $V$. Assume that the network has $L\ge 2$ layers, with connections only between adjacent layers, and has $W$ weights. Assume finally that the neurons in the first layer have non-decreasing activation functions. Let $F$ denote the family of functions on $[0,1]^d$ implementable by such a network.

For any finite subset $S\subset [0,1]^d$ consider the restriction $F|_S$ as a subset of $\mathbb R^{|S|}$ equipped with the uniform norm $\|\cdot\|_\infty$.  We define the \emph{covering number} $N_\infty(\epsilon, F, S)$ as the smallest number of $\epsilon$-balls in $\mathbb R^{|S|}$ covering the set $F|_S$. Then, for any integer $m>0$, we define the covering number $N_\infty(\epsilon, F, m)=\max_{S\subset [0,1]^d, |S|=m}N_\infty(\epsilon, F, m).$ We then have the following bound.

\begin{theorem}[Theorem 14.5 of \cite{anthony2009neural}] $N_\infty(\epsilon,F,m)\le \big(\tfrac{4embW(aV)^L}{\epsilon (aV-1)}\big)^W.$
\end{theorem}

To obtain the desired bound on approximaton rates for H\"older balls $F_{r,d}$, we can now lower-bound $N_\infty(\epsilon,F,m)$ using the $\epsilon$-capacity of H\"older balls. Specifically, observe that the H\"older ball $F_{r,d}$ contains a set $\Phi_\epsilon$ of at least $M_\epsilon=2^{c_{r,d}\epsilon^{-d/r}}$ functions separated by $\|\cdot\|_\infty$-distance $4\epsilon$ (with some constant $c_{r,d}>0$). These functions can be constructed by a standard argument in which we choose in $[0,1]^d$ a grid $S_\epsilon$ of size $c_{r,d} \epsilon^{-d/r}$ (with a spacing $\sim \epsilon^{1/r}$), and then place a properly rescaled spike function with the sign $+$ or $-$ at each point of the grid. The functions of $\Phi_\epsilon$ are mutually $4\epsilon$-separated when restricted to the grid $S_\epsilon$. If our family $F$ of network-implementable functions can $\epsilon$-approximate any function from the balls $F_{r,d}$, then any $\epsilon$-net for $F|_{S_\epsilon}$ is a $2\epsilon$-net for $\Phi_\epsilon|_{S_\epsilon}$, and thus must contain at least $M_\epsilon$ elements. Hence, $M_\epsilon\le N_\infty(\epsilon,F,S_\epsilon)\le N_\infty(\epsilon,F,c_{r,d}\epsilon^{-d/r}),$ i.e. 
\begin{equation}\label{eq:ccrd}
    c_{r,d}\epsilon^{-d/r}\le W\log_2 \big(\tfrac{4e c_{r,d}\epsilon^{-d/r-1}bW(aV)^L}{aV-1}\big).
\end{equation}
Assuming that  $1/\epsilon, W, L, V$ grow while the other parameters are held constant, this bound implies that $$\epsilon\ge c_{r,d,a,b}(WL)^{-r/d}\ln^{-r/d} V$$ with some $c_{r,d,a,b}>0.$

Now suppose that $V$ is a function of $W$, i.e. the magnitude of the weights is allowed to depend on the network size. Suppose that the network achieves the approximation rate $p$, i.e. 
\begin{equation}\label{eq:ccw}
   \epsilon \le C_{r,d,a,b} W^{-p}.
\end{equation} 
Since $L\le W$, comparing Eq.\eqref{eq:ccrd} with Eq.\eqref{eq:ccw}, we then find that 
\begin{equation}\label{eq:clb}
   \ln V\ge c'_{r,d,a,b} W^{pd/r-2}.
\end{equation}
Thus, the rates $p>\tfrac{2r}{d}$ require $V$ to very rapidly grow with $W$. This observation agrees with the main result of Section \ref{sec:otheractiv} -- Theorem \ref{th:sin} -- describing approximation with arbitrary rates $p$ by networks with a periodic activation function. In the proof of this theorem, the network weights are defined with the help of rapidly growing constants $a_k$ given in Eq.\eqref{eq:la}. In particular, we have $\log a_K\sim 2^K$ with $K\sim W^{1/2},$ which agrees with the lower bound \eqref{eq:clb}.

\section{Theorem \ref{th:sin}: sketch of proof}\label{sec:sketchsin}

We can assume without loss of generality  that the period $T=2$ and $\max_{x\in\mathbb R}\sigma(x)=-\min_{x\in\mathbb R}\sigma(x)=1$ (these values can always be effectively adjusted in each neuron by rescaling the input and output weights). We divide the proof into three steps. 

\paragraph{Step 1: reduction to patch-encoders and patch-classifiers.} Recall the concepts of coarser partition on the scale $\tfrac{1}{N}$ and the finer partition on the scale $\tfrac{1}{M}$ used in the proofs of Theorem \ref{th:deepphase} and \ref{th:constwidth}. In those theorems, both $N$ and $M$ were $\sim W^a$ with some constant powers $a$. In contrast, we choose now $N=1$, and we'll set $M$ to grow much faster (roughly exponentially) with $W$: this will be possible thanks to the much more efficient decoding available with the $\sin$ activation.

Specifically, note first that we can implement an almost perfect approximation of the parity function $\theta: x\mapsto (-1)^{\lfloor x\rfloor}$ using a constant size networks, by computing $a\sigma(x)$ with a large $a$ and then thresholding the result at 1 and $-1$ using ReLU operations (the approximation only fails in small neighborhoods of the integer points). If the cube $[0,1]^d$ is partitioned into cubic $M$-patches, we can apply rescaled versions of $\theta$ coordinate-wise to create a binary dictionary of these patches. Specifically, we can construct a network of size $\sim d\log_2 M$ that maps a given $\mathbf x\in [0,1]^d$ to a size-$K$ binary sequence encoding the place of the patch $\Delta_M\ni\mathbf x$ in the cube $[0,1]^d$, with $K\sim d\log_2 M$. We call this network the \emph{patch-encoder}.  

Given a function $f\in F_{r,d},$ we approximate it by a function $\widetilde f$ which is constant in each $M$-patch. Suppose for simplicity and without loss of generality that the smoothness $r\le 1,$ then this approximation has accuracy $\epsilon\sim M^{-r}$. Let $\widetilde f_{\Delta_M}$ be the value that the approximation returns on the patch $\Delta_M$. It is sufficient to define $\widetilde f_{\Delta_M}$ with precision $\sim M^{-r}$. Consider the binary expansion of $f_{\Delta_M}$ that provides this precision: $\widetilde f_{\Delta_M}=-1+\sum_{k=0}^R \widetilde f_{\Delta_M,k}2^{-k},$ where $R\sim r\log_2 M$ and $\widetilde f_{\Delta_M,k}\in\{0,1\}$. Suppose that for each $k$ we can construct a network that maps each patch $\Delta_M$ to the corresponding bit $\widetilde f_{\Delta_M,k}$. Summing these \emph{patch-classifiers} with coefficients $2^{-k}$, we then reconstruct the full approximation $\widetilde f$. 

We have thus reduced the task to efficiently implementing an arbitrary binary classifier on the $M$-partition of $[0,1]^d.$ The patch-encoder constructed above efficiently encodes each $M$-patch by a binary $K$-bit sequence. We can then think of the classifier as an assignment $A:\{0,1\}^K\to\{0,1\}$ that must be implemented by our network. We show below in Step 2 that this can be done by a size-$O(K)$ network, with the assignment encoded in a single weight $w_A$. The full number of network weights (including the patch-encoder and the patch-classifiers on all $R$ scales) can then be bounded by $W=O(KR),$ i.e. $W=O(rd\log_2^2 M)$. The relations $\epsilon\sim M^{-r}$ and $W\sim rd\log_2^2 M$ then yield $\epsilon\sim 2^{-c'W^{1/2}}$ (with $c'\sim\sqrt{r/d}$), as claimed in Eq.\eqref{eq:ratesin}. 

Note, however, that the proof strategy that we have described requires the network to have $R$ $f$-dependent encoding weights (one per each patch-classifier), while Statement 2 of the theorem claims a unique $f$-dependent weight. In Step 3, we will resolve this issue by showing that these $R$ weights can be decoded from a single weight with a subnetwork of size $O(R)$.

To make these arguments fully rigorous, we need to handle the issue of our approximation to the parity function $\theta$ becoming invalid near the boundaries of the patches.  This is done in Section \ref{sec:proofsin} using partitions of unity; the resulting complications do not affect the asymptotic. 

\paragraph{Step 2: implementation of a patch-classifier.} We explain now how an arbitrary assignment $A:\{0,1\}^K\to\{0,1\}$ can be implemented by a network of size $O(K)$ with a single encoding weight $w_A$. Let us define two sequences, $a_k$ and $l_k$:
\begin{equation}\label{eq:la}l_1=\tfrac{1}{2}, \quad a_1=2,\quad l_k=\min(\tfrac{l_{k-1}}{2},\tfrac{l_{k-1}}{a_kc_\sigma}),\quad a_k=\tfrac{4}{l_{k-1}},\end{equation}
where $c_\sigma$ is the Lipschitz constant of $\sigma$. Consider iterations $g_1\circ g_2\circ\ldots \circ g_K(w_*),$
in which each $g_k$ can be either the identity function $g_k(w)=w$, or $g_k(w)=\sigma (a_k w),$ with  some initial value $w_*$. For each $\mathbf z\in\{0,1\}^K$, let us define $H_{K, w_*}(\mathbf z)$ as the $\operatorname{sgn}$ of the value obtained by substituting the respective functions:
$$
H_{K, w_*}(\mathbf z)=\operatorname{sgn}\circ\begin{cases}\operatorname{Id},& z_1=0,\\ 
\sigma(a_1\cdot), & z_1=1\end{cases}\circ
\begin{cases}\operatorname{Id},& z_2=0,\\ 
\sigma(a_2\cdot), & z_2=1\end{cases}\circ
\ldots\circ
\begin{cases}\operatorname{Id},& z_K=0,\\ 
\sigma(a_K\cdot), & z_K=1\end{cases} (w_*) $$
\begin{lemma}\label{lem:sindecoding}
For any assignment $A:\{0,1\}^K\to \{0,1\}$ there exists $w_A\in\mathbb R$ such that $H_{K,w_A}(\mathbf z)=A(\mathbf z)$ for all $\mathbf z\in \{0,1\}^K$.
\end{lemma}
\begin{proof}
Proof by induction on $K$, but of a slightly sharper statement: the desired $w_A$ not only exist, but fill (at least) an interval $I_{K}\subset [-1,1]$ of length $l_K$.

The base $K=1$ follows immediately from the 2-periodicity of $\sigma$ and the  hypothesis that $\sigma(x)>0$ for $x\in [0,1]$ while $\sigma(x)<1$ for $x\in[1,2]$. Suppose we have proved the statement for $K-1$. Given an assignment $A:\{0,1\}^K\to \{0,1\}$, consider it as a pair of assignments $A_0:\{0,1\}^{K-1}\to \{0,1\},A_1:\{0,1\}^{K-1}\to \{0,1\}.$ By the induction hypothesis, we can find two intervals $I_{K-1}^{(0)}$ and $I_{K-1}^{(1)}$ of length $l_{K-1}$ such that $H_{K-1,w_0}(\mathbf z)=A_0(\mathbf z)$  and $H_{K-1,w_1}(\mathbf z)=A_1(\mathbf z)$ for all $w_0\in I_{K-1}^{(0)},w_1\in I_{K-1}^{(1)}$ and $\mathbf z\in \{0,1\}^{K-1}$. Consider the set 
\begin{equation}\label{eq:iw}I=\{w\in\mathbb R:w\in I_{K-1}^{(0)}\text{ and }\sigma(a_K w)\in I_{K-1}^{(1)}\}.\end{equation}
Then for any $w\in I$, we have the desired property $H_{K,w}(\mathbf z)=A(\mathbf z),\forall \mathbf z\in \{0,1\}^K$. We need to show now that $I$ contains an interval of length $l_K$. Observe that, by the relation $a_K=\tfrac{4}{l_{K-1}}$ from Eq.\eqref{eq:la}, the length $l_{K-1}$ of the interval $I^{(0)}_{K-1}$ is twice as large as the period $\tfrac{2}{a_K}$ of the function $\sigma(a_K\cdot).$  Using the assumption that $\max \sigma(x)=-\min\sigma(x)=1$, we see that the function $\sigma(a_K\cdot)$ attains both values 1 and -1 on its period. It follows then by continuity of $\sigma$ that there exists a point $w'$ at a distance not more than $\tfrac{l_{K-1}}{4}$ from the center of $I^{(0)}_{K-1}$ such that $\sigma(a_Kw')$ attains any given value from the interval $[-1,1]$. Let this value be the center  $w''$ of the interval $I_{K-1}^{(1)}$. Since the function $\sigma(a_K\cdot)$ is Lipschitz with constant $a_Kc_\sigma$, we have $|\sigma(a_Kw)-\sigma(a_Kw')| < \tfrac{l_{K-1}}{2}$ for any $w$ such that $|w-w'|<\tfrac{l_{K-1}}{2a_Kc_{\sigma}}$. Then it follows from the definition $l_K=\min(\tfrac{l_{K-1}}{2},\tfrac{l_{K-1}}{a_Kc_\sigma})$ that the length-$l_K$ interval centered at $w'$ is contained in $I$ given by Eq.\eqref{eq:iw}.
\end{proof}
This lemma shows that the network can implement any classifier $A$ if the network can somehow branch into applying either $\operatorname{Id}$ or $\sigma(a_k\cdot)$ depending on the signal bit $b\in\{0,1\}$ that is output by the patch-encoder subnetwork. This branching can be easily implemented by forming the linear combination $(1-b)x+b\sigma(a_kx)$, and also noting that a product of any $x\in\{0,1\}$ and $y\in[-1,1]$ admits the ReLU implementation $xy= \max(0,2x+y-1)-x$.

\paragraph{Remark.} The construction in Lemma \ref{lem:sindecoding} can be interpreted as a dichotomy-based lookup if we think of the assignment $A$ as a binary sequence of size $S=2^K.$ In each of the network steps we divide the sequence in half, ultimately locating the desired bit in $K\sim\log_2 S$ steps. We can compare this with the less efficient bit extraction procedure of \cite{bartlett1998almost} (for which it is however sufficient to only have the ReLU activation in the network). In this latter procedure, the bits are extracted from the encoding weight one-by-one, and so the lookup requires $\sim S$ steps. 

\paragraph{Step 3: ensuring a unique $f$-dependent weight.} Steps 1 and 2 have shown that the desired network can be constructed using at most $R$ $f$-dependent weights, say $w^*_1,\ldots,w^*_R$. We observe now that these values can be approximated arbitrarily well by a serial application of the rescaled activation $\sigma$: 

\begin{lemma}
Let $w^*_1,\ldots,w^*_R\in[-1,1].$ Fix $a>0$ and consider the sequence $w_1',\ldots,w_R'$ defined by $w_k'=\sigma(aw'_{k-1})$ with some initial $w_1'$. Then we can find an initial $w_1'\in[-1,1]$ such that $|w^*_k-w_k'|<\tfrac{2}{a}$ for all $k=1,\ldots,R.$ 
\end{lemma}
\begin{proof}
We use induction on $R$. The base $R=1$ is trival. Suppose we have already proved the statement for $R-1$. Let $\widetilde{w}$ be the corresponding initial value for the sequence $\widetilde{w},\sigma(a\widetilde{w}),\sigma(a\sigma(a\widetilde{w})),\ldots$ approximating the sequence $w^*_2,\ldots,w^*_R$. The function $\sigma(a\cdot)$ has period $\tfrac{2}{a},$ and attains all values from $[-1,1]$ on any interval of this length. It follows that we can find $w_1'$ such that $\sigma(aw_1')=\widetilde{w}$ and $|w^*_1-w_1'|<\tfrac{2}{a}$. This gives the desired $w_1'$.
\end{proof}
Thus, by taking some sufficiently large ($W$-dependent) $a$, we can  generate all the $R$ encoding weights $w^*_k$ with sufficient accuracy from a single weight by using an $f$-independent network of complexity $O(R)$, which is within the desired bound \eqref{eq:ratesin}.

In Figure \ref{fig:deepFourier} we show the overall network layout.  

\begin{figure}
\centering
\input{deep_Fourier_flowchart.tex}
\caption{The network layout overview for the ``deep Fourier'' approximation (see Section \ref{sec:sketchsin}).}
\label{fig:deepFourier}
\end{figure}
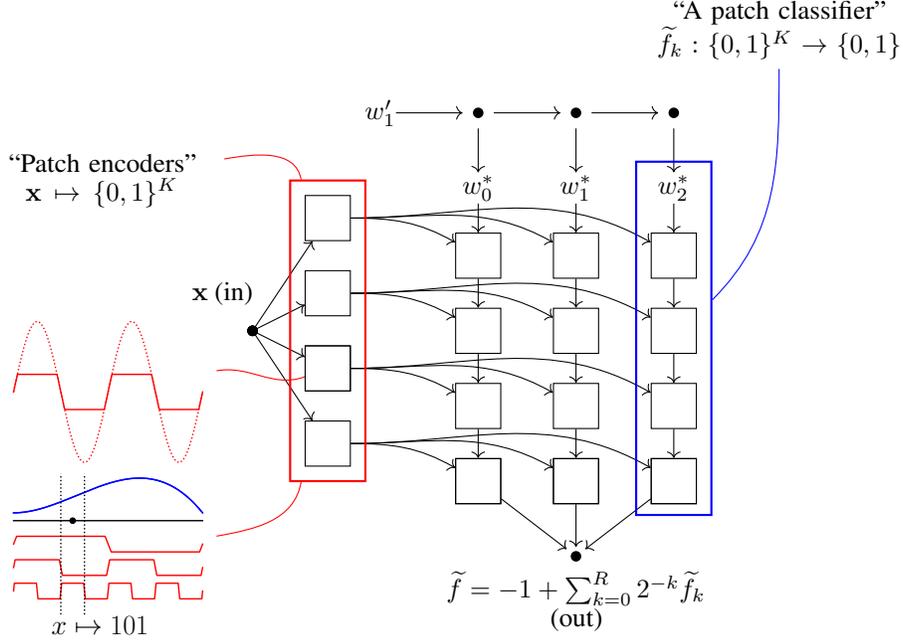

\paragraph{Information in the encoding weight.} Let us make a rough estimate of the amount of information contained in the constructed encoding weight. First, we can estimate it as $R$ (the number of patch classifiers) times the information in the weight $w_A$ corresponding to a single patch classifier (see Lemma \ref{lem:sindecoding}). We have $R\sim r\log_2 M\sim \log_2 \epsilon.$ The information in $w_A$ can be roughly estimated as $-\log_2 l_K$, where $l_K$ is the length of the interval $I_K$ appearing in the proof of Lemma \ref{lem:sindecoding}. From relations \eqref{eq:la}, for small $l_{k-1}$, by combining $l_k=\tfrac{l_{k-1}}{a_kc_\sigma}$ and $a_k=\tfrac{4}{l_{k-1}}$ we get $l_k=\tfrac{l^2_{k-1}}{4c_\sigma}$, which leads to $\log_2 l_K\sim 2^K.$ Since $K\sim  d\log_2 M\sim \tfrac{d}{r}\log_2 \epsilon,$ this gives $-\log_2 l_K\sim \epsilon^{-d/r}$. Summarizing, the total information can be roughly estimated as $\epsilon^{-d/r}\log(1/\epsilon).$

\section{Theorem \ref{th:sin}: proof details}\label{sec:proofsin}
Examining the sketch of proof given in Section \ref{sec:sketchsin}, we see that the only significant gap in the given argument is the treatment of boundaries of the patches. Namely, recall that we use approximations to the parity function $\theta(x)=(-1)^{\lfloor x\rfloor}.$ The approximations can be defined by a finite expression in terms of linear, ReLU and $\sin$ operations: 
$$\widetilde\theta_a(x)=\min(1, \max(-1, a\sigma( x))).$$
By taking $a$ large, we can make $\widetilde\theta_a$ to equal $\theta$ outside some small neighborhood of $\mathbb Z$. Now, recall that we choose patches $\Delta_M$ as cubes $[\tfrac{m_1}{M}, \tfrac{m_1+1}{M}]\times[\tfrac{m_2}{M}, \tfrac{m_2+1}{M}]\times\ldots[\tfrac{m_d}{M}, \tfrac{m_d+1}{M}].$ Assume without loss of generality that $M=2^U$ with some integer $U$. The patch-encoding functions $g_{u,k}:\mathbf x\mapsto\widetilde\theta_a(2^u x_k)$ (with $u=1,2,\ldots,U$ and $k=1,2,\ldots,d$) map the cubes $\Delta_M$ to the values $\pm 1$ everywhere except near the boundaries of these cubes. If we could slightly ``shrink'' the cubes $\Delta_M$ so that they were disjoint, we could adjust $a$ in $\widetilde\theta_a$ so that the functions $g_{u,k}$  were perfectly equal to $\pm 1$ on the whole cubes. The remaining construction of patch-classifying networks in Section \ref{sec:otheractiv} then becomes fully functional and yields the desired asymptotic relation \eqref{eq:ratesin}.

Thus, we need to show how to reduce the problem to the case of disjoint patches. This can be done by using suitable filtering functions, similarly to the proofs of Theorems \ref{th:deepphase} and \ref{th:constwidth}. Fix some $a_0>1$ and consider the functions $\Psi_0, \Psi_1:\mathbb R\to [0,1]$ defined by
$$\Psi_0(x)=\tfrac{1}{2}(1+\widetilde\theta_{a_0}(2Mx)) ,\quad \Psi_1=1-\Psi_0.$$
The functions $\Psi_0$ and $\Psi_1$ form a a two-element partition of unity. Furthermore, since $a_0>1,$ there is $\delta>0$ such that
\begin{align}
\Psi_0(x)=0&\text{ for }x\in(\tfrac{3}{4M}-\delta,\tfrac{3}{4M}+\delta)+\mathbb Z/M,\label{eq:psi}\\ 
\Psi_1(x)=0&\text{ for } x\in(\tfrac{1}{4M}-\delta,\tfrac{1}{4M}+\delta)+\mathbb Z/M.\label{eq:psi2}
\end{align}
Taking the product of the partitions of unity over the $d$ coordinates, we can write for  $\widetilde f:[0,1]^d\to\mathbb R$:
$$\widetilde f=\sum_{\mathbf q\in\{0,1\}^d}(\prod_{s=1}^d\Psi_{q_s})\widetilde f.$$
Thanks to Eqs.\eqref{eq:psi},\eqref{eq:psi2}, for each $\mathbf q\in\{0,1\}^d$, the filtering function $\prod_{s=1}^d\Psi_{q_s}$ vanishes in $[0,1]^d$ outside an $\tfrac{1}{M}$-grid of disjoint cubic patches, exactly as desired. We can then look for the approximation $\widetilde f$ in the form 
$$\widetilde f=\sum_{\mathbf q\in\{0,1\}^d}(\prod_{s=1}^d\Psi_{q_s})\widetilde f_{\mathbf q},$$
where $\widetilde f_{\mathbf q}$ has the required values only on the patches $[0,1]^d\setminus\operatorname{supp}(\prod_{s=1}^d\Psi_{q_s})$ and can be constructed as described in Section  \ref{sec:otheractiv}.

Having implemented these approximations $\widetilde f_{\mathbf q}$, the final approximation is obtained by implementing approximate products with the filters $\Psi_{q_s}$ and performing summation over $\mathbf q\in\{0,1\}^d$. As shown in \cite[Proposition~3]{yarsawtooth}, multiplication with accuracy $\epsilon$ requires a ReLU subnetwork with $O(\log(1/\epsilon))$ connections. This is asymptotically negligible compared to our bound for the total complexity of the patch-classifiers (which is $O(\log^2(1/\epsilon))$).  

%% file: encoding.tikz
\begin{tikzpicture}
\begin{axis}
[
    xmin=0, xmax=15,
    ymin=-1.5, ymax=8,
    xlabel=$\x$,
    %ylabel=$r$,
    ytick=\empty,
    width=15cm,
    height=8cm,
    xtick={7,10},
    xticklabels={$\m_{\n, t} / M$,$\m_{\n, t+1} / M$},
]

\draw[name path=lower, color=blue, domain={0:15}] plot (\x, {47 * \x^3 / 2100 - 373 * \x^2 / 700 + 733 * \x / 210} ) node[pos=0.8, below] {$f(x) =x$};
\draw [dashed] (7,-1.5) -- (7,8);
\draw [decorate,decoration={brace,amplitude=10pt},xshift=-4pt,yshift=0pt]
(7,4) -- (7,6) node [black,midway,xshift=-1.1cm] {\footnotesize $\leq M^{|\k| - r}$};
\draw (7,4) node[cross=0.1cm,red, line width=0.1cm] {};
\node[coordinate, pin={[pin edge={red}]210:{\color{red}$\a_{\m_{\n, t}, \k}$}}] at (axis cs:7,4){};
\draw (7,6) node[cross=0.1cm,blue, line width=0.1cm] {};
\node[coordinate, pin={[pin edge={blue}]30:{\color{blue}$a_{\m_{\n, t}, \k}$}}] at (axis cs:7,6){};

\draw [dashed] (10,-1.5) -- (10,8);
\draw (10,-0.3) node[cross=0.1cm,brown, line width=0.1cm] {};
\draw [decorate,decoration={brace,amplitude=10pt},xshift=-4pt,yshift=0pt]
(10,-0.3) -- (10,4) node [black,midway,xshift=-1.2cm] {\footnotesize $< 4 M^{|\k| - r}$};
\draw (10,1.8) node[cross=0.1cm,red, line width=0.1cm] {};
\draw [decorate,decoration={brace,amplitude=10pt, mirror},xshift=4pt,yshift=0pt]
(10,-0.3) -- (10,1.8) node [black,midway,xshift=1.4cm] {\footnotesize $M^{|\k| - r} B_{\n, \k, t}$};
\draw [decorate,decoration={brace,amplitude=10pt, mirror},xshift=4pt,yshift=0pt]
(10,1.8) -- (10,4) node [black,midway,xshift=1.1cm] {\footnotesize $\leq M^{|\k| - r}$};
\node[coordinate, pin={[pin edge={red}]0:{\color{red}$\a_{\m_{\n, t+1}, \k}$}}] at (axis cs:10,1.8){};
\draw (10,4) node[cross=0.1cm,blue, line width=0.1cm] {};
\node[coordinate, pin={[pin edge={blue}]30:{\color{blue}$a_{\m_{\n, t+1}, \k}$}}] at (axis cs:10,4){};
\node[coordinate, pin={[pin edge={brown}]330:{\color{brown}$\at_{\m_{\n, t+1}, \k}$}}] at (axis cs:10,-0.3){};

\node[coordinate, pin={[pin edge={blue}]330:{\color{blue}$D^{\k} f(x)$}}] at (axis cs:2,176/35){};
\draw[brown, dashed] plot [smooth,tension=1]
        coordinates {(7,4) (8,0.5) (10,-0.3)}
        [arrow inside={end=stealth,opt={brown,scale=2}}{0.75}];

\end{axis}
\end{tikzpicture}

%% file: deep_Fourier_flowchart.tex
\begin{tikzpicture}[scale=1]
\def\r{0.07}

\coordinate (out) at (4.3,-3);
\fill  (out) circle (\r);
\node[rectangle,minimum size=2mm] (out1) at (4.3,-3) {}; 

\node at (-0.4,0.5) {$\mathbf x$ (in)};
\node[text width=3.5cm,align=center] at (4.3,-3.6) { $\widetilde f =-1+\sum_{k=0}^R 2^{-k}\widetilde f_k$\\ (out)};

\foreach \k in {-2,...,1}{
 \node[rectangle,draw,minimum size=6mm] (B) at (1,\k+0.5) {};
 \foreach \s in {0}{
   \coordinate (input) at (0,0.5*\s);
   \fill  (input) circle (\r);
   \draw[->] (input) -- (B);
  }
  
  \foreach \q in {0,...,2}{
   \node[rectangle,draw,minimum size=6mm] (D) at (3+1.3*\q,\k) {};
   \draw[->] (B) to [out=0,in=150] (D);
   \node[rectangle,minimum size=6mm] (D0) at (3+1.3*\q,\k+1) {}; 
   \draw[->] (D0) -- (D); 
  };       
}

\foreach \q in {0,...,2}{
  \node[rectangle,minimum size=4mm] (D0) at (3-1.3+1.3*\q,2.9) {};  
  \coordinate (D3) at (3+1.3*\q,2.9);
  \fill  (D3) circle (\r);
  \node[rectangle,minimum size=4mm] (D1) at (3+1.3*\q,2.9) {};
  \node[rectangle,minimum size=6mm] (D2) at (3+1.3*\q,1.8) {};
  \draw[->] (D0) to [out=0,in=180] (D1);
  \draw[->] (D1) to [out=-90,in=90] (D2);
  \node (D) at (3+1.3*\q,1.9) {$w^*_\q$};
};

\node at (3-1.3,2.9) {$w'_1$};

\foreach \q in {0,...,2}{
  \node[rectangle,draw,minimum size=6mm] (D) at (3+1.3*\q,-2) {};
  \draw[->] (D) to (out1);
    
}

\node[text width=3cm,align=center] (Pencoder) at (-2,2)  {``Patch encoders''\\ $\mathbf x\mapsto \{0,1\}^K$ };
\node (rectEncoder) at (1,0) [draw,red, thick,minimum width=1cm,minimum height=4cm] {};
\draw[red] (Pencoder) to[out=10, in=100] (rectEncoder);

%\draw[red,thick] (0.5,-2) rectangle (1.5,2);
\node  at (-2,-0.8) (osc) {\includegraphics[width=0.2\textwidth, trim={10mm 10mm 12mm 10mm}, clip]{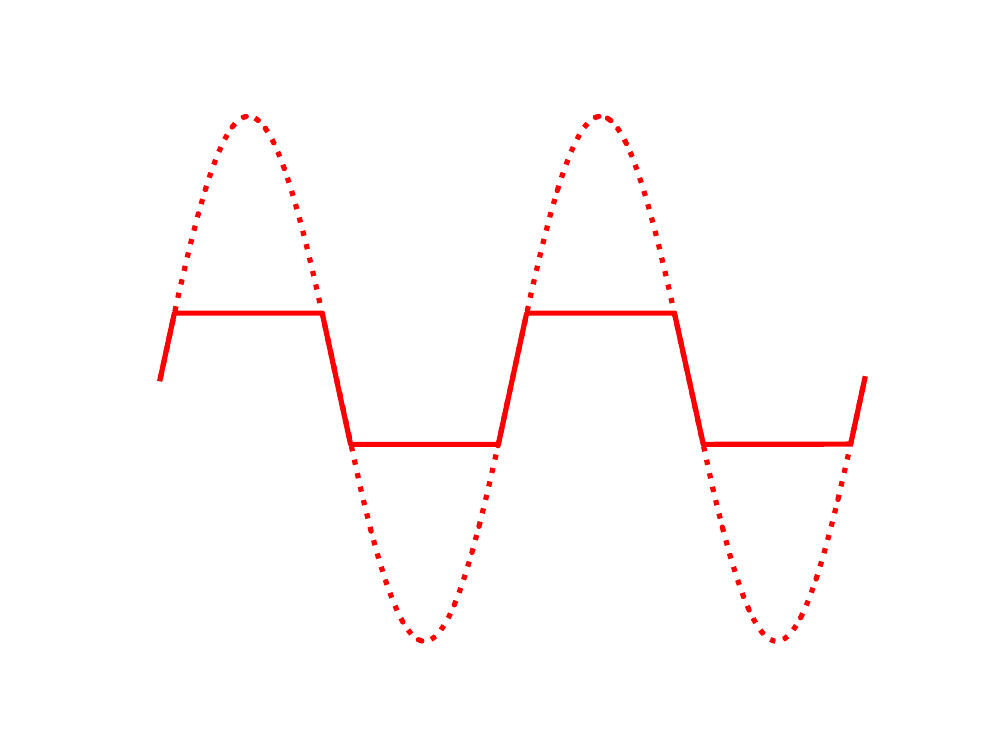}};
\node[rectangle,draw,minimum size=6mm] (B1) at (1,-0.5) {};
\draw[red] (osc) to[out=10, in=-160] (B1);
\node  at (-2,-3) (osc2) {\includegraphics[width=0.2\textwidth, trim={10mm 0mm 12mm 10mm}, clip]{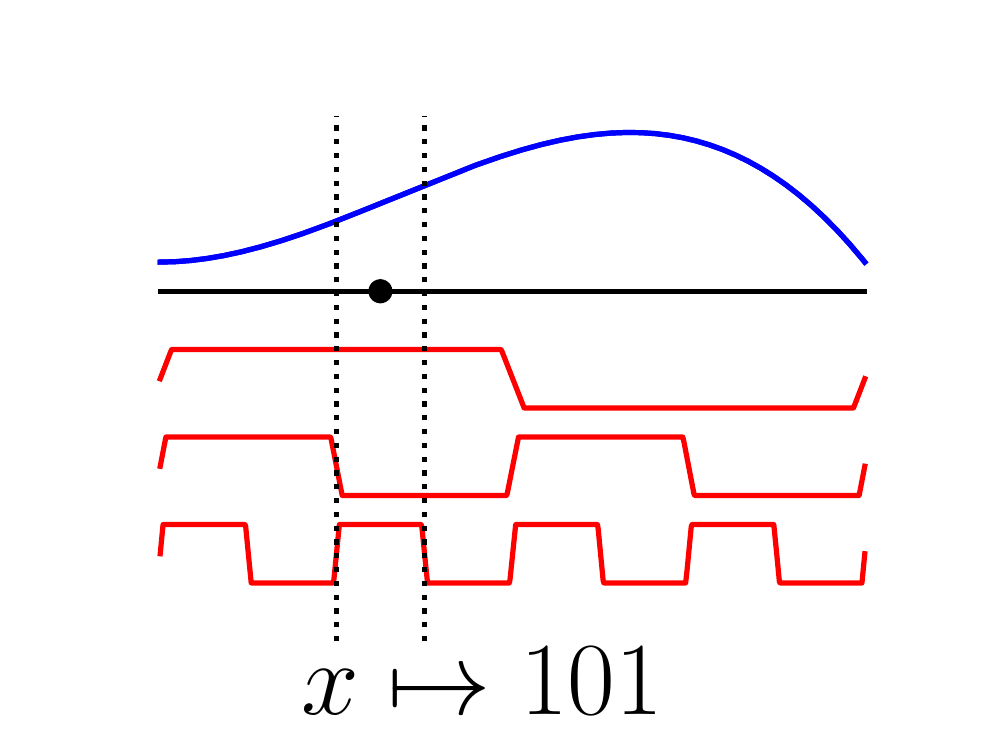}};
\draw[red] (osc2) to[out=10, in=-100] (rectEncoder);

\node[text width=3.5cm,align=center] (Pclassifier) at (7.0,4.0)  {``A patch classifier''\\ $\widetilde f_k:\{0,1\}^K\to \{0,1\}$ };
\node (rectClassifier) at (3+2*1.3,-0.1) [draw,blue, thick,minimum width=1cm,minimum height=4.7cm] {};
\draw[blue] (Pclassifier) to[out=-90, in=45] (rectClassifier);

\end{tikzpicture}